\documentclass[11pt,reqno]{amsart}
\usepackage[margin=1in]{geometry}
\usepackage[T1]{fontenc}
\usepackage{lmodern,amsmath,amssymb,amsthm,mathtools,microtype,enumitem,aliascnt,xcolor}
\usepackage[colorlinks=true,linkcolor=blue!50!black,citecolor=blue!50!black,urlcolor=blue!50!black]{hyperref}
\hypersetup{pdftitle={A positive resolution of the gap-entropy conjecture}}
\usepackage{zref-clever}
\zcsetup{cap=true,abbrev=false,nameinlink=true}
\allowdisplaybreaks[1]
\newtheorem{theorem}{Theorem}[section]
\newaliascnt{lemma}{theorem}
\newtheorem{lemma}[lemma]{Lemma}
\aliascntresetthe{lemma}
\newaliascnt{corollary}{theorem}
\newtheorem{corollary}[corollary]{Corollary}
\aliascntresetthe{corollary}
\newaliascnt{proposition}{theorem}
\newtheorem{proposition}[proposition]{Proposition}
\aliascntresetthe{proposition}
\numberwithin{equation}{section}
\newcommand{\E}{\mathbb E}
\newcommand{\Pp}{\mathbb P}
\DeclareMathOperator{\Ent}{Ent}
\DeclareMathOperator{\KL}{KL}
\DeclareMathOperator{\kl}{kl}
\DeclareMathOperator{\Mean}{Mean}
\DeclareMathOperator{\Median}{Median}
\DeclareMathOperator{\Fraction}{Fraction}
\DeclareMathOperator{\Eliminate}{Eliminate}
\newcommand{\LL}{\mathcal L}
\title{A positive resolution of the gap-entropy conjecture}
\author{P. M. Aronow, Nathan Kallus, and Patrick Lopatto}
\date{September 9, 2026}
\thanks{This paper was written
by the authors with the assistance of large language models, which were used to
suggest mathematical arguments, draft and revise exposition, and write code
for computational verification. P.\ L.\ was partially supported by NSF grant DMS-2450004.}
\begin{document}

\begin{abstract}
We prove the gap-entropy conjecture for fixed-confidence best-arm identification with independent
unit-variance Gaussian arms, means in $[0,1]$, and a unique optimal arm. For each suboptimal arm $i$,
let $\Delta_i=\mu_*-\mu_i$ be its gap from the optimal mean, and write
$H=\sum_{i\ne *}\Delta_i^{-2}$. Let $p_r$ be the fraction of $H$ contributed by arms with
$2^{-(r+1)}<\Delta_i\le2^{-r}$, and let
$\Ent(I)=\sum_{r:p_r>0} p_r\log(1/p_r)$. Among all algorithms that identify the optimal arm with probability
at least $1-\delta$ on every Gaussian instance, the optimal expected number of samples on a given
instance, averaged over all permutations of the arm labels, is within absolute constant factors of
$H(\log(1/\delta)+\Ent(I))$. Moreover, there is an algorithm, independent of the instance, whose
expected number of samples is bounded by a constant multiple of this quantity plus
$g^{-2}\log\log(e^e/g)$, where $g=\min_{i\ne *}\Delta_i$ is the gap to the closest competitor.
\end{abstract}
\maketitle

\section{Introduction}\label{sec:introduction}

Experiments often compare several alternatives in order to choose the one with the best mean reward. Adaptive sampling, wherein at each step an algorithm pulls one of $K\geq2$ arms and immediately receives a sample from its unknown reward distribution to inform future pulls, can significantly improve performance. In regret minimization, avoiding apparently inferior arms in future sampling can help attain higher cumulative rewards over the experiment, and in best-arm identification, concentrating future sampling on apparently close competitors for best can find the best faster.
Fixed-confidence best-arm identification formalizes this problem by seeking to identify the best arm with as few samples as possible while controlling the probability of returning the wrong arm for all admissible reward distributions \cite{EMM06,KCG16}.

The difficulty depends on the gaps between the mean rewards $(\mu_1,\dots,\mu_K)\in[0,1]^K$, where $\mu_*=\max_i\mu_i$ is uniquely attained. In the unit-variance Gaussian model, a suboptimal arm $i\neq *$ with gap $\Delta_i=\mu_*-\mu_i$ from the optimal arm requires $\Omega(\Delta_i^{-2}\log(1/\delta))$ samples to distinguish it from the optimal arm at error level $\delta$. Summing
these costs gives the lower bound $\Omega(H\log(1/\delta))$, where $H=\sum_{i\neq*}\Delta_i^{-2}$ is the sum of the inverse
squared gaps, capturing the baseline cost of distinguishing the best arm if we know up to orders of magnitude how well separated it is from each other arm \cite{KCG16,MT04}. This does not, however, reveal the cost of discovering this scale of the gaps, which matters even when the error level $\delta$ is fixed.

Exponential-Gap Elimination
\cite{KKS13} and lil'UCB \cite{JMNB14} achieved upper bounds with iterated-log dependence on gaps of the form $O(H\log(1/\delta)+\sum_{i\neq*}\Delta_i^{-2}\log\log(e^e/\Delta_i))$. Chen and Li \cite{CL15} subsequently improved this bound, capping each arm's iterated-logarithmic factor at $\log\log K$, apart from a $g^{-2}\log\log(1/g)$ term where $g=\min_{i\neq *}\Delta_i$. They also proved that for sufficiently small $\delta$ and large $K$, any $\delta$-correct algorithm has a Gaussian instance requiring $\Omega(H\log\log K)$ samples. This raised the question of how this unavoidable adaptation cost depends on the particular configuration of gaps.

Chen and Li \cite{CL16} proposed that the additional instance-dependent cost is described by
\emph{gap entropy}. For an instance $I$, peel the suboptimal arms into geometrically spaced shells according to their gaps, and assign to each shell the fraction of $H$ contributed by its arms. The entropy of the resulting probability vector, $\mathrm{Ent}(I)$, measures how widely the required sampling effort is distributed across scales.

Chen and Li conjectured (their conjecture 3.5) that, among $\delta$-correct algorithms, the smallest expected sample size achievable on a given instance $I$ (averaged over permutations of its arm labels) is of order $H(\log(1/\delta)+\mathrm{Ent}(I))$. (A different algorithm may attain the benchmark for each $I$, provided it remains $\delta$-correct on all instances.) They further conjectured (their conjecture 3.2) that a single algorithm can attain this instance-wise benchmark simultaneously on all instances, up to an additive term of order $g^{-2}\log\log(1/g)$. Such an adaptation cost is unavoidable already with two arms: adapting to an unknown gap that may be arbitrarily small incurs an iterated-logarithmic cost along some sequences of instances, even though this cost is not required at every fixed instance \cite{CL16,Farrell64}.

The conjectures seek nonasymptotic bounds that hold uniformly in the instance and $\delta$. By contrast,
Track-and-Stop \cite{GK16} already achieves asymptotic optimality as $\delta\downarrow0$ with the instance fixed. In that limit, the gap entropy of the fixed instance is
eventually dominated by $\log(1/\delta)$. An asymptotically optimal leading constant therefore does
not determine the additional cost when the gap configuration and the confidence parameter vary together. Simchowitz, Jamieson, and Recht \cite{SJR17} studied the complexity distinctions between fixed $\delta$ and $\delta\downarrow0$ asymptotics, including the cost of learning a sampling allocation in structured models. Cho and Kallus \cite{CK25} relax exact error control to asymptotic validity in the $\delta\downarrow0$ regime, enabling the use of strong invariance principles to match the optimal exact-error Gaussian benchmark with known variances even under nonparametric reward distributions.

Chen, Li, and Qiao \cite{CLQ17} made progress toward the conjectures, obtaining the upper bound $O(H(\log(1/\delta)+\Ent(I))+g^{-2}\log\log(1/g)\;\mathrm{polylog}(K,1/\delta))$ (their theorem 1.11) and, for Gaussian instances whose gaps are powers of two, a lower bound of $\Omega(H(\log(1/\delta)+\mathrm{Ent}(I)))$ for algorithms whose expected sample cost does not increase when suboptimal arms are deleted (their theorem 1.12 and corollary 1.13).

In this paper, we answer both conjectures in the positive. We remove both the extra polylogarithmic factor from the upper bound (which applies to arbitrary $1$-sub-Gaussian reward distributions) and the restrictions on instances and algorithms in the lower bound (for the Gaussian model).

\subsection{Main results}
For $K\ge2$, a $K$-armed bandit instance $I=(\nu_1,\ldots,\nu_K)$ has reward distributions
$\nu_i$ with means $\mu_i\in[0,1]$. The index of the largest mean, denoted by $*$, is assumed unique. Each arm has an sequence of independent samples with law
$\nu_i$, and the sequences are independent across arms. Pulling an arm reveals its next
sample.

Let $\mathcal S_K$ be the class of such unique-optimum instances with $1$-sub-Gaussian rewards, meaning that
\[
 \E_{X\sim\nu_i}\exp\bigl(\lambda(X-\mu_i)\bigr)
 \le \exp(\lambda^2/2)
 \qquad (\lambda\in\mathbb R,\ i\in[K]).
\]
The algorithm knows that the rewards are $1$-sub-Gaussian, but not their distributions or means.
Let $\mathcal G_K\subset\mathcal S_K$ be the Gaussian subclass with $\nu_i=N(\mu_i,1)$.

For either class, write
\[
 m=\mu_*,\qquad \Delta_i=m-\mu_i\ (i\ne *),\qquad
 H=\sum_{i\ne *}\Delta_i^{-2},\quad
 g=\min_{i\ne *}\Delta_i,\quad D=g^{-2}.
\]
For $r=0,1,\ldots$, let
\[
 G_r=\{i\ne *:2^{-(r+1)}<\Delta_i\le2^{-r}\},\qquad
 H_r=\sum_{i\in G_r}\Delta_i^{-2},\qquad p_r=H_r/H.
\]
For a probability vector $q\in[0,1]^{\mathbb N_0}$, write $\Ent(q)=\sum_{r:q_r>0}q_r\log(1/q_r)$ for its
Shannon entropy. The \emph{gap entropy} of instance $I$ is defined as
\begin{equation}\label{eq:gap-entropy}
 \Ent(I)=\Ent(p).
\end{equation}
It is zero when all gaps lie on the same geometric scale. If $d$ groups contribute equally to
$H$, it equals $\log d$. More generally, writing $\mathcal U=\{r:H_r>0\}$ for the finite set of
indices of the nonempty gap groups, we have $0\le\Ent(I)\le\log|\mathcal U|$.

All logarithms are natural unless otherwise indicated. Constants denoted by $c$ or $C$ are positive
and absolute and may change between occurrences. Constants used as algorithm parameters are fixed
within each construction. We write $[K]=\{1,\ldots,K\}$.

An algorithm is $\delta$-correct on a model class if, on every instance in that class, it stops
almost surely and returns the optimal arm with probability at least $1-\delta$. Write $T_A$ for
its total number of samples. For a permutation $\pi\in S_K$ of the labels, let $\pi I$ be the
relabeled instance. For $I\in\mathcal G_K$, following \cite[definition 3.1]{CL16}, define
\begin{equation}\label{eq:instance-complexity}
 \LL(I,\delta)=\inf_{A:\,\delta\text{-correct on }\mathcal G_K}
 \frac1{K!}\sum_{\pi\in S_K}\E_{\pi I}T_A.
\end{equation}
The averaging removes any advantage from a favorable labeling. The infimum is taken separately for
each Gaussian instance $I$, but every algorithm in it must be $\delta$-correct on the Gaussian class
$\mathcal G_K$. 
\zcref{app:conventions} compares these conventions with those of \cite{CL16} and explains why our results imply their original conjectures.

We now present our main results.

\begin{theorem}\label{thm:instance}
There exist absolute constants $c,C>0$ such that, for every $I\in\mathcal G_K$ and
$0<\delta<0.1$,
\[
 cH(\log(1/\delta)+\Ent(I))\le\LL(I,\delta)
 \le CH(\log(1/\delta)+\Ent(I)).
\]
\end{theorem}

The lower and upper bounds are proved in \zcref{sec:lower,sec:instance-upper}, respectively. The lower bound is proved directly for the original Gaussian instance, after averaging over
permutations of the arm labels, without first passing to a subinstance obtained by deleting arms.  The matching upper bound uses an algorithm selected for that target instance.
More generally, \zcref{prop:subgaussian-instance-upper} constructs, for each instance in
$\mathcal S_K$, an algorithm that is $\delta$-correct on $\mathcal S_K$ and attains the same bound on every relabeling of the instance. The next theorem gives one algorithm for the
whole sub-Gaussian class.

\begin{theorem}\label{thm:uniform}
There exists one algorithm, given only $K$ and $\delta$, that is $\delta$-correct on
$\mathcal S_K$ and satisfies, for every $I\in\mathcal S_K$ and $0<\delta<0.1$,
\[
 \E_I T_A\le C\left(H(\log(1/\delta)+\Ent(I))
             +D\log\log(e^e/g)\right).
\]
The constant $C$ is absolute.
\end{theorem}

The algorithm is defined in \zcref{sec:algorithm}; its correctness and
bounds on the expected sample size are proved in
\zcref{sec:uniform-analysis,sec:expected-cost}, respectively. The additive term depends only on the smallest gap, with no further factor involving $K$ or
$\delta$. Specializing the uniform algorithm to $\mathcal G_K$ and applying
\zcref{thm:instance} gives the near-optimality on every individual instance conjectured in
\cite[conjecture 3.2]{CL16}.

\begin{corollary}\label{cor:almost}
For every $I\in\mathcal G_K$ and $0<\delta<0.1$, the algorithm in \zcref{thm:uniform} satisfies
\[
 \E_I T_A\le C(\LL(I,\delta)+D\log\log(e^e/g)).
\]
\end{corollary}

The proof is given at the end of \zcref{sec:expected-cost}. \zcref{thm:instance} proves \cite[conjecture 3.5]{CL16}, and together with \zcref{cor:almost} proves
\cite[conjecture 1.9]{CLQ17}; see
\zcref{app:conventions} for details.

\subsection{Ideas of the proof}

A short allocation calculation explains why entropy appears. Suppose the group weights $H_r$ were known and group $G_r$ were assigned error budget
$\delta_r$. The corresponding sampling cost would have the form $\sum_r H_r\log(1/\delta_r)$. Under the constraint
$\sum_r\delta_r\le\delta$, this expression is minimized by $\delta_r=\delta p_r$, giving
\begin{equation}\label{eq:entropy-allocation}
 \inf_{\substack{\delta_r>0\ (r\in\mathcal U)\\
                  \sum_{r\in\mathcal U}\delta_r\le\delta}}
       \sum_{r\in\mathcal U}H_r\log(1/\delta_r)
 =H(\log(1/\delta)+\Ent(I)).
\end{equation}
In particular, if $d$ groups contribute equally to $H$, then $\delta_r=\delta/d$ and the
additional cost is $H\log d$. 

This calculation, used to motivate the conjecture in
\cite[section 4]{CL16}, does not by itself prove a lower bound for an arbitrary algorithm or give
an algorithm when the weights are unknown. We first show that every algorithm correct on the
Gaussian class must pay the entropy cost. We then attain this cost using an algorithm selected for
a fixed target instance. Finally, we construct an algorithm that does not depend on the instance
and bound the additional cost of adaptation.

\subsubsection{The lower bound}
The term $H\log(1/\delta)$ follows from the standard bandit change-of-measure
argument; see \cite[lemma 1]{KCG16}. The remaining task is to
establish the additional entropy cost for an arbitrary correct algorithm. For each gap group $G_r$, we choose one of its arms uniformly and examine how many times the
algorithm samples that arm. Since arms in $G_r$ have gaps of order $2^{-r}$, their sampling complexity is on the scale $4^r$, and we associate $G_r$ with an event in which this sample count lies
in a suitable range around that scale. Widely separated scales give disjoint ranges, and the main difficulty is to compare the
probabilities of these events under one common reference law.

Fix a Gaussian instance and an arbitrary algorithm $A$ that is $\delta$-correct on the Gaussian
class. Define a new algorithm by first applying a uniformly random permutation to the arm labels,
running $A$ on the relabeled instance, and then applying the inverse permutation to $A$'s output.
The new algorithm is still $\delta$-correct, treats all labelings symmetrically, and has expected
sample cost on every labeling equal to the average cost of $A$ over all labelings of the original
instance.  Let $u_i$ be its expected number of samples from arm $i$ of the target instance, and let
\[
 \alpha_r=\frac{\sum_{i\in G_r}u_i}{|G_r|4^r},\qquad r\in\mathcal U.
\]
An arm in $G_r$ has inverse squared gap comparable to $4^r$. Thus $\alpha_r$ is the average sample
count in the group divided by this sampling scale.

Fix $a\ge1$. For each group with $\alpha_r\le a$, let $P_r$ be the law of the algorithm's run
when the optimal arm has label $1$, a uniformly chosen arm from $G_r$ has label $2$, and the
remaining arms are uniformly relabeled. Let $N_2$ count the samples from label $2$, and let
$E_r(a)$ be the event that the algorithm returns label $1$ with
$c_04^r\le N_2\le4a4^r$, for a sufficiently small absolute constant $c_0$.
Since $\E_{P_r}N_2=\alpha_r4^r\le a4^r$, Markov's inequality gives
$P_r(N_2>4a4^r)\le1/4$.

For the lower tail, raise label $2$'s mean to $m$, producing a tied reference law $Q_r$.
Correctness on nearby instances in which label $2$ is uniquely optimal forces the probability
of returning label $1$ under $Q_r$ to be at most $\delta$. Before label $2$ has been sampled $c_04^r$ times, the KL divergence between the corresponding
history distributions is bounded by an absolute constant times $c_0$. A finite-horizon
change-of-measure argument therefore shows that returning label $1$ this early also has small
probability under $P_r$. Combining the two tail bounds with correctness gives
$P_r(E_r(a))\ge1/2$; the details are in \zcref{sec:lower}.

For sufficiently separated indices $r,r'$, the ranges $[c_04^r,4a4^r]$ and
$[c_04^{r'},4a4^{r'}]$ are disjoint. But their event probabilities have been bounded under different laws $P_r$, and the tied
reference laws $Q_r$ also depend on $r$. To compare different groups, we construct a reference law
that can be used for all of them. Fix a nonempty group $G_s$. For $r>s$,
let $i$ denote the arm placed at label $2$, choose $j$ uniformly from
$G_s$, independently of the labeling, and change the two means by 
\[
 (\mu_i,\mu_j)\longmapsto(m,\mu_i).
\]
The second change restores the suboptimal mean removed by the first. Labels $1$ and $2$ are now
tied, and the other labels carry exactly the original suboptimal means with $\mu_j$ omitted. 
After averaging over the random choices, those remaining means are uniformly arranged on labels
$3,\ldots,K$. The resulting law, denoted by $Q_s$, therefore depends on $s$ but not on $r$.
For $r=s$, raising the selected arm to $m$ alone gives the same $Q_s$. Note that each of these resulting instances is inadmissible because labels 1 and 2 are tied for best. While $\delta$-correctness or almost-sure stopping are therefore not required for $A$ on these, the law of the algorithm's run is still well-defined and its finite-time answer probabilities are controlled by its required $\delta$-correct behavior on nearby admissible instances where label 1's mean is slightly reduced, making label 2 uniquely optimal.

A Gaussian change-of-measure calculation bounds the KL divergence from $P_r$ to $Q_s$,
at any finite horizon, by $(\alpha_r+\alpha_s)/2$. 
If any group satisfies $\alpha_r\le a$, choose
$s=\min\{r\in\mathcal U:\alpha_r\le a\}$. The divergence bound is then at most $a$ for every
such group, and $P_r(E_r(a))\ge1/2$ implies
$Q_s(E_r(a))\ge e^{-2a}/4$. Under this one law, events from sufficiently separated groups are disjoint. Since each has
probability at least $e^{-2a}/4$, any such collection contains at most $4e^{2a}$ groups.
Partitioning the indices into \(O(a)\) separated collections gives the required counting bound for
\(a\ge1\). The standard per-arm lower bound gives \(\alpha_r>1\) for every \(r\), so integrating
this counting bound, as carried out in \zcref{sec:lower}, yields
\[
 \sum_r e^{-8\alpha_r}\le1.
\]
Since the weights $e^{-8\alpha_r}$ have total mass at most one, normalizing them and applying
Gibbs's inequality gives
\[
 H\Ent(I)\le8\sum_r H_r\alpha_r\le32\sum_{i\ne *}u_i.
\]
Since $\sum_{i\ne *}u_i$ is the expected number of samples from the suboptimal arms, it is at
most the expected total number of samples. Hence, writing $T$ for the total number of samples
\[
 \E T\ge \frac1{32}H\Ent(I).
\]
Together with the standard lower bound $\E T\ge cH\log(1/\delta)$, this gives
\[
 \E T\ge c'H(\log(1/\delta)+\Ent(I))
\]
for an absolute constant $c'>0$, proving the lower half of \zcref{thm:instance}.

\subsubsection{The instance-wise upper bound}
For each fixed target instance $I\in\mathcal S_K$, we may choose an algorithm specifically for $I$,
provided it remains $\delta$-correct on all of $\mathcal S_K$. The algorithm's definition includes the indices and sizes of the target's nonempty gap groups
and a threshold 
$\tau\in(m-g/2,m-g/3)$. Thus $\tau$ lies strictly between the largest mean $m$ and the
second-largest mean $m-g$, at distance comparable to $g$ from both. On the target, the optimal arm is the unique arm with mean above $\tau$. Thus it suffices to eliminate all arms whose means lie below $\tau$. We do this through a
sequence of independent trials. Each trial may return an arm or return without an answer.

Choose probability weights $q_r$ proportional to $|G_r|4^r$, so that $q_r$ is comparable
to $p_r$. Given a parameter $\eta$ that will serve as the trial's error allowance, the trial starts with
all arms active and cycles through the active arms, taking one sample from each in turn. For each arm, we combine one exponential test statistic
for each supplied gap scale, using mixture weight $q_r$ for scale $r$. An arm is eliminated
when this combined statistic reaches its rejection threshold. 
For a suboptimal target arm $i\in G_r$, let $U_i$ be the number of samples from that arm
before it is eliminated. We show that
\[
 \E U_i\le
 C\Delta_i^{-2}\bigl[\log(1/\eta)+\log(1/q_r)\bigr].
\] 
Thus the mixture weight $q_r$ enters the sample bound through the additive term $\log(1/q_r)$.
Since $q_r$ is comparable to $p_r$, summing over the suboptimal arms gives
\[
 \sum_{i\ne *}\E U_i
 \le CH(\log(1/\eta)+\Ent(I)).
\] 
Because the procedure takes one sample from every active arm in each cycle through the active
set, the optimal arm is sampled at most $\max_{i\ne *}U_i$ times before at most one arm remains.  Hence the total number of samples used during elimination is at most
\[
 \sum_{i\ne *}U_i+\max_{i\ne *}U_i
 \le 2\sum_{i\ne *}U_i.
\]

Although the algorithm is chosen using information from the target instance, it must remain
$\delta$-correct when run on every instance in $\mathcal S_K$, including instances whose means
and gap groups differ from those of the target. There are two cases. If the input has an arm with mean at least $\tau$, then the optimal arm also has mean at least
$\tau$. An incorrect output is therefore possible only if the tests eliminate the optimal arm,
which happens with probability at most $\eta$. If instead every input mean is below $\tau$, then before returning the sole remaining arm the
trial takes fresh samples from it and requires its empirical mean to exceed a level slightly
above $\tau$. We choose the number of confirmation samples so that, by the sub-Gaussian tail bound, this event
has probability at most $\eta$. Thus in this case the trial returns any arm with probability at
most $\eta$, and in particular returns an incorrect arm with probability at most $\eta$.

We impose a deterministic sample limit on each trial and repeat independent trials with
geometrically decreasing error allowances whose sum is at most $\delta/2$. On the target instance, each
trial returns the optimal arm with probability at least $3/4$, so the probability of reaching
later trials decreases geometrically while their sample limits grow only linearly with the trial
index. The repeated-trial procedure need not stop on every other instance. To guarantee
almost-sure stopping on the full class, we interleave it with an independent
$\delta/2$-correct confidence-interval procedure and return the first answer. This at most
doubles the target-instance cost, up to one sample, while the two error guarantees together give
$\delta$-correctness.

\subsubsection{The uniform upper bound}

The preceding instance-wise construction uses the instance's gap distribution to choose its testing
scales and error allocation. We now construct one algorithm, given only \(K\) and \(\delta\), that
must discover this information from the data. It runs independent stages with geometrically
increasing budgets $
 B_t=100^t$. 
Each stage begins with all \(K\) arms and either returns an arm or rejects and passes control to the
next stage. A deterministic sample limit of order \(B_t\) bounds the cost of the stage on every
sample path.

Within a stage, round \(r\) works at accuracy $
 \varepsilon_r=2^{-r}$. 
A Median call selects an arm whose mean is within \(\varepsilon_r/8\) of the largest active mean,
with probability at least \(0.99\). The algorithm estimates this arm's mean and uses a Fraction
call to determine whether sufficiently many active arms lie substantially below it. If so, an
Eliminate call removes arms with low estimated means; otherwise the active set is unchanged. The
stage returns as soon as only one arm remains. It rejects if either of its two budget tests is
triggered, its final permitted round is completed without a singleton, or its sample limit is
reached.

For an active set \(S_r\), the sampling-cost scale at accuracy \(\varepsilon_r\) is
\[
 w_r=|S_r|\varepsilon_r^{-2}=|S_r|4^r.
\]
This observable quantity replaces the unknown gap-group weights used by
the instance-wise algorithm. Set \(\eta=c\delta\), where \(c>0\) is a
sufficiently small absolute constant. An Eliminate call in round \(r\)
receives an error allowance proportional to \(w_r/B_t\). The stage
controls both the sum of these allowances and the cumulative sum of
\[
 z_r=w_r\log\frac{B_t}{\eta w_r}.
\]
Up to an absolute constant, \(z_r\) bounds the expected cost of the
Median and Eliminate calls in that round.

The main cost estimate controls how long each suboptimal arm remains
active. In the comparison process of \zcref{sec:uniform-analysis},
if Median succeeds in round \(r\) and every Eliminate call through that
round is valid, the next active set has size at most an absolute
constant times
\[
 1+|\{i\ne *:\Delta_i\le2\varepsilon_r\}|.
\]
If Median fails for \(j\) additional rounds, the probability of that delay is at most
\(0.01^j\), whereas the sampling scale grows by \(4^j\). The series
\[
 \sum_{j\ge0}(4\cdot0.01)^j
\]
is finite, so these delays change the expected cost by only an absolute factor. Consequently, arms with gaps of order \(2^{-r}\) contribute mainly near
round \(r\). For \(B_t\ge H\), sum \(z_r\) over executed rounds up to and
including the first invalid Eliminate call, or until an earlier stage
termination. The active-set and entropy estimates in
\zcref{lem:uniform-charging} give
\[
 \E\sum_r z_r
 \le CH\left[\log(1/\delta)+\Ent(I)+1+\log(B_t/H)\right].
\]

The outer Mean and Fraction calls require a separate summation. The separation between the two
fraction thresholds is indexed by the number of rounds remaining. Reindexing from the final round
turns the resulting polynomial factor into a summable correction to the geometric round costs; see
\zcref{lem:uniform-aux}. The resulting sample costs also grow logarithmically with the stage
index. The calculation in \zcref{lem:uniform-success} accounts for
this dependence through the term \(CH(\log(1/\delta)+\Ent(I))\)
together with
\[
 CD\log\log(e^e/g).
\]

Correctness across stages requires more than a direct union bound,
because each stage receives the same total budget for elimination
errors. We again use the comparison process. For a budget \(B\), call
the optimal arm and every suboptimal arm satisfying
\(\Delta_i^{-2}>B\) hard; call the remaining suboptimal arms easy.
A wrong output requires the optimal arm and all but possibly one of
the hard suboptimal arms to be removed. Every removal of a hard
suboptimal arm occurs before its gap scale is reached. The product
guarantee for Eliminate controls these joint premature removals.
When there are few hard arms, a localized estimate instead uses
\[
 B^{-1}\sum_{i:\,i\text{ easy}}\Delta_i^{-2}.
\]
When this ratio is large, an exponential-moment argument shows that
acceptance itself is unlikely. Together with the bound on disagreement
between the actual and comparison processes, these estimates control
the total error across stages; see \zcref{sec:uniform-analysis}.

Finally, \zcref{lem:uniform-success} shows that a stage accepts with probability at least
\(199/200\) once
\[
 B_t\ge
 C\left\{
 H(\log(1/\delta)+\Ent(I))
 +D\log\log(e^e/g)
 \right\}.
\]
Beyond this point, the probability of reaching the next stage decreases by a factor at most
\(1/200\), while its sample limit increases by a factor of \(100\). The successive expected costs
therefore decrease by a factor at most \(1/2\). This proves both almost-sure stopping and the
claimed bound on the expected sample size.  The construction modifies
\cite[algorithms 4--5]{CLQ17} and is specified in \zcref{sec:algorithm}; its correctness and cost
are proved in \zcref{sec:uniform-analysis,sec:expected-cost}.

\section{The instance-wise lower bound}\label{sec:lower}
Throughout this section the instance belongs to $\mathcal G_K$, and $\delta$-correctness is
required on $\mathcal G_K$. We prove the lower bound in \zcref{thm:instance}. The term $H\log(1/\delta)$
follows from the usual Gaussian change of measure. For the entropy term, we compare events defined by the number of samples from a selected arm across
different gap groups, using a common law with two tied optimal arms.

The conversion of these sample-count bounds into an entropy lower bound parallels the entropy lower-bound argument of Chen, Li and Qiao \cite[lemma 4.1 and appendix D]{CLQ17}. The additional issue here is to
construct a common reference law while evaluating all expected sample counts on permutations of the
original instance.

The proof of the following proposition appears at the end of this section. 

\begin{proposition}\label{prop:explicit-lower}
For every $I\in\mathcal G_K$ and $0<\delta<0.1$, write
$b_\delta=\kl(1-\delta,\delta)$. Then
\begin{equation}\label{eq:lower-stronglower}
 \LL(I,\delta)\ge\max\{2b_\delta H,\ H\Ent(I)/32\},
\end{equation}
where $\kl(p,q)=p\log(p/q)+(1-p)\log((1-p)/(1-q))$.
\end{proposition}

Since $b_\delta=(1-2\delta)\log((1-\delta)/\delta)\asymp\log(1/\delta)$
uniformly over the stated confidence range, the lower half of \zcref{thm:instance} follows.

\subsection{Symmetry and changes of measure}
Fix $I\in\mathcal G_K$ and let $A$ be an arbitrary algorithm that is
$\delta$-correct on $\mathcal G_K$.
Randomly permuting labels before running $A$, and undoing the permutation on its answer, gives a
permutation-equivariant algorithm $B$ such that
\begin{equation}\label{eq:lower-sym}
 \E_{\pi I}T_B=\frac1{K!}\sum_{\rho\in S_K}\E_{\rho I}T_A
 \quad\text{for every }\pi.
\end{equation}
The algorithm $B$ is still $\delta$-correct. It is useful to distinguish an arm's identity from its
label: the identities refer to the fixed instance, whereas the labels specify where the arms are
presented to the algorithm. Distinct identities are retained even when their means coincide. Let
$u_i$ be the expected number of samples from identity $i$ under $B$. By permutation equivariance, the expected number of samples allocated to a given arm identity
does not depend on the label at which that identity is placed. Thus, for every permutation
$\pi$,
\[
 \E_{\pi I}N_{\pi(i)}=u_i.
\] 
Consequently, averaging over any of the random labelings used below does not alter these
expectations.

If the right side of \eqref{eq:lower-sym} is infinite there is nothing to prove. Henceforth assume
$\E_I T_B<\infty$, and suppress the subscript $B$.

We use the standard finite-horizon likelihood-ratio calculation underlying the bandit
change-of-measure inequality; see \cite[lemma 1 and appendix A.1]{KCG16}. For a deterministic integer $M$, write $P_\mu^{(M)}$
for the distribution of the recorded history up to $T\wedge M$ under mean vector $\mu$.
For two fixed mean vectors $\mu,\lambda$, the chain rule gives
\begin{equation}\label{eq:lower-chain}
 \KL(P_\mu^{(M)}\Vert P_\lambda^{(M)})
 =\frac12\sum_i(\mu_i-\lambda_i)^2
          \E_\mu N_i(T\wedge M).
\end{equation}
Here $N_i(t)$ counts samples from the arm at label $i$ by time $t$. Write $N_i=N_i(T)$
for its total number of samples, and let $\widehat i$ denote the returned label. On
$\{T=\infty\}$, set $\widehat i=0$ and interpret $N_i(T)$ as
$\lim_{t\to\infty}N_i(t)$. The history at time $T\wedge M$ records whether the algorithm
has stopped and, if so, its output, as well as whether truncation occurred. Since the
algorithm uses the same sampling, randomization, and stopping rules under both models,
these decisions do not contribute  additional KL divergence.

We will also use \eqref{eq:lower-chain} after truncation at bounded stopping times and after
introducing auxiliary random variables having the same law under both models; these follow
from the same chain-rule calculation and data processing. Whenever we use a tied reference
law, we first work at a fixed finite horizon and only then pass to the limit. Thus we never
need to assume that the algorithm stops when the two arms are tied.

We first derive the standard lower bound for the expected number of samples from a single
suboptimal arm; see \cite[lemma 1]{KCG16}. We claim that 
\begin{equation}\label{eq:lower-perarm}
 u_i\ge2b_\delta\Delta_i^{-2},\qquad
 \E_I T\ge2b_\delta H.
\end{equation}
For an alternative instance with a different optimal
arm, apply data processing to the event that the algorithm returns the original optimal arm to obtain
\begin{equation}\label{eq:lower-ordinaryKL}
 \frac12\sum_i u_i(\mu_i-\lambda_i)^2\ge b_\delta.
\end{equation}
(To justify \eqref{eq:lower-ordinaryKL} from \eqref{eq:lower-chain}, first use the event
$\{T\le M,\widehat i=*\}$ and then let $M\to\infty$.) If $m<1$, raise suboptimal arm $i$ to $m+\varepsilon$
and send $\varepsilon\downarrow0$ in \eqref{eq:lower-ordinaryKL} to get \eqref{eq:lower-perarm}. If $m=1$, instead raise arm $i$ to
$1$ and lower the original optimal arm to $1-\varepsilon$, with $0<\varepsilon<g$, and take $\varepsilon \rightarrow 0$ in \eqref{eq:lower-ordinaryKL}, using $u_*<\infty$.

For each nonempty gap group $r\in\mathcal U$, put
\begin{equation}\label{eq:lower-alpha}
 n_r=|G_r|,\qquad t_r=\sum_{i\in G_r}u_i,
 \qquad \alpha_r=\frac{t_r}{n_r4^r}.
\end{equation}
Since $\Delta_i^{-2}\ge4^r$ in $G_r$, \eqref{eq:lower-perarm} implies
\begin{equation}\label{eq:lower-alphaone}
 \alpha_r\ge2b_\delta>1.
\end{equation}

\subsection{A common reference law}
For $r\in\mathcal U$, let $P_r$ be the law of the algorithm's run under
the following random relabeling of the original instance $I$, with all
arm distributions left unchanged: the optimal arm has label $1$; an identity $i$ chosen uniformly
from $G_r$ has label $2$; the other identities are placed uniformly on labels $3,\ldots,K$. Then
\begin{equation}\label{eq:lower-Pr}
 P_r(\widehat i=1)\ge1-\delta,\qquad
 \E_{P_r}N_2=\alpha_r4^r.
\end{equation}
The expectation identity follows from permutation equivariance, since the algorithm is not told which
instance identities occupy labels $1$ and $2$.

Under $P_r$, raising only the mean at label $2$ to $m$ would produce a
tied law whose remaining suboptimal arms omit the chosen identity from
$G_r$; this law would therefore still depend on $r$. To obtain a common
reference law, fix $s\in\mathcal U$ and define $Q_s$ as follows. Give
labels $1$ and $2$ mean $m$, choose an identity $j$ uniformly from
$G_s$, and assign all original suboptimal identities other than $j$
uniformly to labels $3,\ldots,K$. Thus $Q_s$ has exactly two tied
optimal arms and does not depend on $r$. Although this tied law lies
outside the unique-optimum model class, its finite-time answer
probabilities are controlled by perturbations with a unique optimum,
as the next lemma shows. Let $P_r^{(M)}$ and $Q_s^{(M)}$ denote the
corresponding laws of the recorded histories up to $T\wedge M$.

\begin{lemma}\label{lem:lower-tie}
For every $s\in\mathcal U$, $
 Q_s(T<\infty,\widehat i=1)\le\delta$.
\end{lemma}
\begin{proof}
Since the optimal arm is unique and all means are nonnegative, $m>0$.
For $0<\varepsilon<m$, let $Q_{s,\varepsilon}$ be obtained from $Q_s$
by lowering the mean at label $1$ from $m$ to $m-\varepsilon$. Every
component of $Q_{s,\varepsilon}$ has label $2$ as its unique optimal
arm. Therefore, for every fixed $M$,
\[
 Q_{s,\varepsilon}(T\le M,\widehat i=1)\le\delta.
\]
The KL divergence between the laws under $Q_{s,\varepsilon}$ and $Q_s$
of the history up to $T\wedge M$ is at most $M\varepsilon^2/2$.
Pinsker's inequality consequently gives
\[
 Q_s(T\le M,\widehat i=1)
 \le \delta+\frac{\varepsilon\sqrt M}{2}.
\]
First let $\varepsilon\downarrow0$ with $M$ fixed, and then let
$M\to\infty$ to obtain the claim.
\end{proof}

\begin{lemma}\label{lem:lower-couple}
For $r,s\in\mathcal U$ with $s\le r$ and every finite horizon $M$,
\begin{equation}\label{eq:lower-couple}
 \KL(P_r^{(M)}\Vert Q_s^{(M)})
 \le\frac{\alpha_r+\alpha_s}{2}.
\end{equation}
When $r=s$, the sharper bound $\alpha_s/2$ holds.
\end{lemma}
\begin{proof}
Suppose first that $s<r$. Start from the random labeling defining
$P_r$: choose $i\in G_r$ uniformly, place it at label $2$, and arrange
the remaining identities uniformly on the other labels. Independently
choose $j\in G_s$ uniformly. Modify two means: raise the mean at label
$2$ from $\mu_i$ to $m$, and change the mean at the label occupied by
$j$ from $\mu_j$ to $\mu_i$. The first change makes labels $1$ and $2$
tied for the largest mean, while the second replaces the omitted
identity $i$ among labels $3,\ldots,K$ by the omitted identity $j$.

We verify that the modified model has law $Q_s$. Conditional on $i$ and
$j$, the identities on labels $3,\ldots,K$ are initially a uniform
arrangement of all suboptimal identities except $i$. Replacing $j$ by
$i$ gives a bijection from these arrangements to the arrangements of
all suboptimal identities except $j$. The latter arrangement is
therefore uniform and independent of $i$ and $r$. Averaging over the
uniform choice of $j\in G_s$ gives $Q_s$. This identity-level argument
also applies when distinct arms have equal means.

Adjoin the choices of $i,j$ and the random labeling to the histories
under both models. Conditional on these variables, the two models
differ only at label $2$ and at the label occupied by $j$.
Equation~\eqref{eq:lower-chain} therefore bounds their conditional KL
divergence by
\[
 \frac12\E_{P_r}\left[
 \Delta_i^2 N_2(T\wedge M)
 +(\Delta_j-\Delta_i)^2
   N_{\operatorname{label}(j)}(T\wedge M)
 \right].
\]
Forgetting the auxiliary variables can only decrease KL divergence.
Since $s<r$, the  definitions of $G_r$ and $G_s$ give
\[
 0<\Delta_i<\Delta_j,\qquad
 \Delta_i^2\le4^{-r},\qquad
 (\Delta_j-\Delta_i)^2\le4^{-s}.
\]
Moreover, equivariance and uniform choice within the two groups give
expected total sample counts $t_r/n_r$ for $i$ and $t_s/n_s$ for $j$.
Hence
\[
 \KL(P_r^{(M)}\Vert Q_s^{(M)})
 \le\frac12\left(
      4^{-r}\frac{t_r}{n_r}
      +4^{-s}\frac{t_s}{n_s}\right)
 =\frac{\alpha_r+\alpha_s}{2}.
\]

When $r=s$, raise only the mean at label $2$ from $\mu_i$ to $m$. The
resulting marginal law is $Q_s$, and
\[
 \KL(P_s^{(M)}\Vert Q_s^{(M)})
 \le\frac12\,4^{-s}\E_{P_s}N_2
 =\frac{\alpha_s}{2}.
\]
\end{proof}

\subsection{Sample counts and entropy}
Recall that under $P_r$, label $2$ contains an arm chosen uniformly
from $G_r$, and
\[
 \E_{P_r}N_2=\alpha_r4^r.
\]
Set $c_0=1/256$. For $a\ge1$ and $r\in\mathcal U$ with
$\alpha_r\le a$, define
\begin{equation}\label{eq:lower-window}
 E_r(a)=\{T<\infty,\ \widehat i=1,\
               c_04^r\le N_2\le4a4^r\}.
\end{equation}

\begin{lemma}\label{lem:lower-window}
If $\alpha_r\le a$, then $
 P_r(E_r(a))\ge 1/2 $.
\end{lemma}
\begin{proof}
Put $L=c_04^r$ and $n=\lceil L\rceil-1$. Since $N_2$ is integer,
\[
 N_2<L\quad\Longleftrightarrow\quad N_2\le n.
\]
Fix a total horizon $M$. Starting from the history under $P_r$, stop
recording at the first of the following times: when the algorithm
stops, when it has taken $M$ samples in total, or immediately after it
selects label $2$ for the $(n+1)$st time and before the corresponding
reward is revealed. Record which of these three stopping conditions
occurred. The resulting history contains at most $n$ rewards from
label $2$ and determines the event
\[
 A_M=\{T\le M,\ \widehat i=1,\ N_2\le n\}.
\]

Compare this censored history under $P_r$ with the corresponding
history under $Q_r$. Conditional on the selected identity
$i\in G_r$ and the random labeling, the two models differ only in the
mean at label $2$, which changes from $\mu_i$ to $m$. Each observed
reward from that label contributes Gaussian KL divergence
$\Delta_i^2/2$, and at most $n$ such rewards are observed. Since
$\Delta_i^2\le4^{-r}$,
\[
 \KL(\overline P_{r,M,n}\Vert\overline Q_{r,M,n})
 \le \frac12n4^{-r}
 \le\frac{c_0}{2},
\]
where $\overline P_{r,M,n}$ and $\overline Q_{r,M,n}$ denote the laws
of the censored histories. The same bound holds after averaging over
the selected identity and labeling, since forgetting these auxiliary
variables can only decrease KL divergence.

By \zcref{lem:lower-tie},
\[
 Q_r(A_M)\le Q_r(T<\infty,\widehat i=1)\le\delta.
\]
Pinsker's inequality therefore gives
\[
 P_r(A_M)
 \le Q_r(A_M)
    +\sqrt{\frac12
       \KL(\overline P_{r,M,n}\Vert\overline Q_{r,M,n})}
 \le\delta+\frac{\sqrt{c_0}}{2}.
\]
Letting $M\to\infty$ and using $c_0=1/256$, we obtain
\[
 P_r(T<\infty,\widehat i=1,N_2<c_04^r)
 \le\delta+\frac1{32}.
\]

For the upper tail, \eqref{eq:lower-Pr} and $\alpha_r\le a$ give
\[
 \E_{P_r}N_2=\alpha_r4^r\le a4^r.
\]
Hence Markov's inequality yields
\[
 P_r(N_2>4a4^r)\le\frac14.
\]
Finally, $P_r(T<\infty,\widehat i=1)\ge1-\delta$. Removing from this
event the lower- and upper-tail events bounded above gives
\[
 P_r(E_r(a))
 \ge1-\delta-\left(\delta+\frac1{32}\right)-\frac14
 =1-2\delta-\frac1{32}-\frac14>\frac12,
\]
because $\delta<0.1$.
\end{proof}

\begin{lemma}\label{lem:lower-kraft}
For $a\ge0$, let
\[
 N(a)=|\{r\in\mathcal U:\alpha_r\le a\}|.
\]
Then
\begin{equation}\label{eq:lower-count}
 N(a)\le32ae^{2a}\qquad(a\ge1),
 \qquad
 \sum_{r\in\mathcal U}e^{-8\alpha_r}\le1.
\end{equation}
\end{lemma}
\begin{proof}
Fix $a\ge1$. If $N(a)=0$, the first bound is immediate. Otherwise let
\[
 s=\min\{r\in\mathcal U:\alpha_r\le a\}.
\]
Every index counted by $N(a)$ satisfies $r\ge s$. Since both
$\alpha_r$ and $\alpha_s$ are at most $a$,
\zcref{lem:lower-couple} gives
\[
 \KL(P_r^{(M)}\Vert Q_s^{(M)})\le a
\]
for every finite horizon $M$.

We use this inequality to transfer the event $E_r(a)$ from $P_r$ to
the common law $Q_s$. Set
\[
 B_{r,M}=E_r(a)\cap\{T\le M\},\qquad
 p_M=P_r(B_{r,M}),\qquad q_M=Q_s(B_{r,M}).
\]
The event $B_{r,M}$ is determined by the history up to $T\wedge M$, so
data processing and the binary-divergence bound
\[
 \kl(p,q)\ge p\log(1/q)-\log2
\]
give
\[
 a\ge\kl(p_M,q_M)
   \ge p_M\log(1/q_M)-\log2.
\]
As $M\to\infty$, $ p_M $ increases to $P_r(E_r(a))$, which satisfies 
\[
P_r(E_r(a))\ge\frac12
\]
by \zcref{lem:lower-window}, while
$q_M$ increases to $Q_s(E_r(a))$. Therefore
\[
 \log\frac1{Q_s(E_r(a))}
 \le2a+2\log2,
\]
and hence
\begin{equation}\label{eq:lower-nullmass}
 Q_s(E_r(a))\ge\frac14e^{-2a}.
\end{equation}

It remains to use the fact that sufficiently separated indices impose
disjoint conditions on the same random variable $N_2$. Let
\[
 d(a)=\left\lceil\log_4(1024a)\right\rceil+1.
\]
If $r'>r$ and $r'\equiv r\pmod{d(a)}$, then
$r'-r\ge d(a)$ and
\[
 c_04^{r'}
 \ge c_04^{r+d(a)}
 >4a4^r,
\]
because $c_0=1/256$ and $4^{d(a)}>1024a$. Thus the intervals
\[
 [c_04^r,\,4a4^r]
\]
are disjoint within each residue class modulo $d(a)$. The corresponding
events $E_r(a)$ are therefore disjoint under the common law $Q_s$.
Since each has probability at least $\frac14e^{-2a}$, each residue
class contains at most $4e^{2a}$ indices counted by $N(a)$. Consequently,
\[
 N(a)\le4d(a)e^{2a}.
\]
Finally,
\[
 d(a)\le\log_4a+7\le8a\qquad(a\ge1),
\]
which proves
\[
 N(a)\le32ae^{2a}.
\]

By \eqref{eq:lower-alphaone}, every $\alpha_r$ is greater than $1$, so
$N(a)=0$ for $a\le1$. Using
\[
 e^{-8\alpha_r}
 =\int_{\alpha_r}^{\infty}8e^{-8a}\,\mathrm da
\]
and summing over $r$ gives
\[
 \begin{aligned}
 \sum_{r\in\mathcal U}e^{-8\alpha_r}
 &=\int_1^\infty8e^{-8a}N(a)\,\mathrm da\le256\int_1^\infty ae^{-6a}\,\mathrm da=\frac{448}{9}e^{-6}<1.
 \end{aligned}
\]
\end{proof}

\begin{proof}[Proof of \zcref{prop:explicit-lower}]
Recall that $T$ denotes the stopping time of the symmetrized algorithm.
Set
\[
 z=\sum_{r\in\mathcal U}e^{-8\alpha_r}.
\]
By \zcref{lem:lower-kraft}, $0<z\le1$, so
\[
 q_r=\frac{e^{-8\alpha_r}}{z},
 \qquad r\in\mathcal U,
\]
defines a probability vector. Gibbs's inequality gives
\[
 \begin{aligned}
 0
 &\le\sum_{r\in\mathcal U}p_r\log\frac{p_r}{q_r}=-\Ent(I)+8\sum_{r\in\mathcal U}p_r\alpha_r+\log z.
 \end{aligned}
\]
Since $\log z\le0$, it follows that
\begin{equation}\label{eq:lower-entropy-alpha}
 \Ent(I)\le8\sum_{r\in\mathcal U}p_r\alpha_r.
\end{equation}

We now convert the right side into an expected sample count. Using
$p_r=H_r/H$ and $\alpha_r=t_r/(n_r4^r)$ gives
\[
 H\sum_{r\in\mathcal U}p_r\alpha_r
 =\sum_{r\in\mathcal U}\frac{H_r}{n_r4^r}t_r.
\]
For every $i\in G_r$, the inequality
$\Delta_i>2^{-(r+1)}$ implies
$\Delta_i^{-2}<4^{r+1}$. Hence
\[
 H_r=\sum_{i\in G_r}\Delta_i^{-2}
 \le4n_r4^r,
\]
and therefore
\[
 H\sum_{r\in\mathcal U}p_r\alpha_r
 \le4\sum_{r\in\mathcal U}t_r
 =4\sum_{i\ne *}u_i
 \le4\E_I T.
\]
Combining this estimate with \eqref{eq:lower-entropy-alpha} yields
\[
 \E_I T\ge\frac{H\Ent(I)}{32}.
\]
The per-arm change-of-measure bound \eqref{eq:lower-perarm} also gives
\[
 \E_I T\ge2b_\delta H.
\]
Finally, \eqref{eq:lower-sym} identifies $\E_I T$ with the average
expected cost of the original algorithm over all relabelings of $I$.
Since that algorithm was arbitrary, taking the infimum in
\eqref{eq:instance-complexity} proves \eqref{eq:lower-stronglower}.
\end{proof}

\section{The instance-wise upper bound}\label{sec:instance-upper}
The following sub-Gaussian proposition implies the upper bound in
\zcref{thm:instance} when restricted to $I\in\mathcal G_K$.

\begin{proposition}\label{prop:subgaussian-instance-upper}
For every $I\in\mathcal S_K$ and $0<\delta<0.1$, there exists an algorithm $A$ that is
$\delta$-correct on $\mathcal S_K$ and satisfies
\[
 \E_{\pi I}T_A\le CH(\log(1/\delta)+\Ent(I))
 \qquad (\pi\in S_K).
\]
The constant $C$ is absolute.
\end{proposition}

The remainder of this section proves \zcref{prop:subgaussian-instance-upper}. We first describe
the target-instance-dependent parameters and a single trial, then prove its correctness and expected-cost
bounds, and finally repeat the trials and add an auxiliary procedure to ensure almost-sure stopping
on every instance in $\mathcal S_K$.

It suffices to prove the upper bounds for $0<\delta<0.01$. For
$0.01\le\delta<0.1$, run the corresponding algorithm with error parameter $0.005$.
It remains $\delta$-correct, while
\[
 \log 200\le
 \frac{\log 200}{\log 10}\log(1/\delta),
\]
so the sample bound changes only by an absolute constant factor.

\subsection{Instance-specific parameters and a trial}
Fix a target instance $I\in\mathcal S_K$ and assume
$0<\delta<0.01$. The algorithm constructed in this section may depend
on the target's unlabeled collection of means, but not on the label of
its optimal arm. For $r\in\mathcal U$, set
\[
 n_r=|G_r|,\qquad w_r=n_r4^r,\qquad
 W=\sum_{r\in\mathcal U}w_r,\qquad q_r=\frac{w_r}{W},
\]
and let
\[
 R=\max\mathcal U,\qquad h=2^{-R}.
\]
For $i\in G_r$,
\[
 4^r\le\Delta_i^{-2}<4^{r+1},
\]
so summing over $G_r$ gives $w_r\le H_r<4w_r$. Consequently,
\begin{equation}\label{eq:lower-upperweights}
 w_r\le H_r<4w_r,\qquad W\le H<4W,\qquad
 g\le h<2g,\qquad h^{-2}\le W.
\end{equation}
Here the bounds involving $h$ follow because an arm with gap $g$
belongs to $G_R$, while $G_R$ is nonempty and therefore
$W\ge w_R\ge4^R=h^{-2}$.

We will allocate the mixture weights according to $q=(q_r)$. To
compare the entropy of these weights with $\Ent(I)$, define, for a
finite nonnegative vector $v$,
\begin{equation}\label{eq:entropy-functional}
 \Phi(v)=\Big(\sum_rv_r\Big)\log\Big(\sum_rv_r\Big)
           -\sum_rv_r\log v_r,
 \qquad 0\log0=0.
\end{equation}
If $V=\sum_rv_r>0$, then
\[
 \Phi(v)=V\Ent(v/V).
\]
Moreover, $\Phi$ is homogeneous of degree one and coordinatewise
nondecreasing: at a positive coordinate,
\[
 \frac{\partial\Phi}{\partial v_r}
 =\log\frac{\sum_s v_s}{v_r}\ge0,
\]
and the same monotonicity at zero follows by continuity. Applying these
properties to the coordinatewise inequalities
$w_r\le H_r\le4w_r$ gives
\begin{equation}\label{eq:lower-entcomp}
 W\Ent(q)=\Phi(w)
 \le\Phi((H_r)_r)=H\Ent(I)
 \le\Phi(4w)=4W\Ent(q).
\end{equation}

Choose a rational threshold
\begin{equation}\label{eq:lower-tau}
 \tau\in(m-g/2,m-g/3)\subset(0,1).
\end{equation}
On the target instance, the optimal mean $m$ lies above $\tau$, while
every suboptimal mean is at most $m-g<\tau$. The algorithm stores the
finite set $\mathcal U$, the integers $(n_r)_{r\in\mathcal U}$, and the
rational number $\tau$. These parameters have a finite description
and do not identify the optimal arm's label.

Fix an error allowance $0<\eta<0.01$. A trial attempts to eliminate
arms whose means are below $\tau$ while retaining any arm whose mean is
at least $\tau$. For each arm $i$, after observing $t$ samples
$X_{i,1},\ldots,X_{i,t}$, define
\begin{equation}\label{eq:lower-mixture}
 M_i(t)=\sum_{r\in\mathcal U}q_r
 \exp\left(
   -\lambda_r\sum_{\ell=1}^t(X_{i,\ell}-\tau)
   -\frac{\lambda_r^2t}{2}
 \right),
 \qquad \lambda_r=2^{-r-2}.
\end{equation}
The component indexed by $r$ is tuned to departures below $\tau$ on
the scale $2^{-r}$. Under the hypothesis $\mu_i\ge\tau$, each component,
and hence their mixture, is a nonnegative supermartingale.

Initialize all arms as active. Repeatedly perform a complete sweep,
drawing one new sample from each arm active at the start of the sweep.
At the end of the sweep, simultaneously discard every arm $i$ whose
updated statistic satisfies $M_i(t)\ge1/\eta$. If no arm remains, end
the trial without an answer. If exactly one arm $c$ remains, take
\[
 n_\eta=\left\lceil288h^{-2}\log(1/\eta)\right\rceil
\]
fresh samples from it. Output $c$ if their average exceeds
$\tau+h/12$; otherwise end the trial without an answer. 
If every input mean lies below $\tau$, this confirmation step bounds
the probability that the trial returns an arm by $\eta$.

Finally, impose the deterministic sample limit
\begin{equation}\label{eq:lower-cap}
 B_\eta=
 \left\lceil
 16000W\bigl[\log(1/\eta)+\Ent(q)+1\bigr]
 \right\rceil.
\end{equation}
The limit includes both elimination and confirmation samples. If the
next sample would make the total exceed $B_\eta$, end the trial without
an answer.

\subsection{Correctness and cost of a trial}
We first show that the trial has error probability at most $\eta$ on
every input instance, even though its parameters were chosen from the
fixed target instance. Let an input arm have mean $\mu\ge\tau$. For
each $r\in\mathcal U$ and a fresh sample $X$ from this arm, the
sub-Gaussian assumption gives
\[
 \begin{aligned}
 \E e^{-\lambda_r(X-\tau)-\lambda_r^2/2}
 &=e^{-\lambda_r(\mu-\tau)}
   \E e^{-\lambda_r(X-\mu)-\lambda_r^2/2}\le e^{-\lambda_r(\mu-\tau)}
 \le1.
 \end{aligned}
\]
The same inequality holds conditionally on all preceding samples.
Hence each exponential process in \eqref{eq:lower-mixture} is a
nonnegative supermartingale. Their weighted sum $M_i$ is therefore
also a nonnegative supermartingale, with
$M_i(0)=\sum_rq_r=1$. The maximal inequality for nonnegative
supermartingales gives
\[
 \Pp\left(\sup_{t\ge0}M_i(t)\ge\frac1\eta\right)\le\eta;
\]
see \cite[theorem 3.9]{LS20}. Thus an arm with mean at least $\tau$ is
eliminated with probability at most $\eta$.

Now consider an arbitrary input instance. If its optimal mean is at
least $\tau$, the trial can return a suboptimal arm only if it first
eliminates the optimal arm. The preceding bound shows that this event
has probability at most $\eta$. If every arm has mean below $\tau$,
condition on the history up to the confirmation step and on the
remaining arm $c$. The confirmation samples are fresh, and
$\mu_c<\tau$, so \eqref{eq:subgaussian-tail} gives
\[
 \Pp\left(
   \overline X_c>\tau+\frac h{12}
   \,\middle|\,\text{preceding history}
 \right)
 \le
 \exp\left(-\frac{n_\eta h^2}{288}\right)
 \le\eta.
\]
The deterministic sample limit can only end the trial
without an answer. Therefore, on every instance in $\mathcal S_K$, the
probability that the trial returns an incorrect arm is at most
$\eta$.

We now return to the target instance $I$ and bound the expected cost.
Temporarily remove the deterministic sample limit, retaining all other
parts of the trial. For each arm, generate an infinite independent
sample sequence and define
\[
 U_i=\inf\left\{t\ge1:M_i(t)\ge\frac1\eta\right\},
\]
with $U_i=\infty$ if the threshold is never reached. This definition
allows us to bound the cost even on paths where the optimal arm is
erroneously eliminated.

Fix a suboptimal target arm $i\in G_r$ and put
\[
 d_i=\tau-\mu_i.
\]
Since $m-\tau<g/2\le\Delta_i/2$,
\[
 d_i=\Delta_i-(m-\tau)>\frac{\Delta_i}{2}.
\]
The definition of $G_r$ and the choice
$\lambda_r=2^{-r-2}$ also give
\[
 \frac{\Delta_i}{4}\le\lambda_r
 <\frac{\Delta_i}{2}<d_i.
\]
The logarithm of the unweighted $r$th exponential component after
$t$ samples is
\[
 a_it-\lambda_r\sum_{\ell=1}^t(X_{i,\ell}-\mu_i),
 \qquad
 a_i=\lambda_rd_i-\frac{\lambda_r^2}{2}.
\]
The preceding inequalities imply
\[
 a_i
 =\lambda_r\left(d_i-\frac{\lambda_r}{2}\right)
 \ge\frac{\Delta_i^2}{16},
 \qquad
 \frac{a_i}{\lambda_r}
 =d_i-\frac{\lambda_r}{2}
 \ge\frac{\Delta_i}{4}.
\]

Set
\[
 b_r=\log\frac1{\eta q_r}.
\]
If $U_i>t$, then $M_i(t)<1/\eta$, so its $r$th component satisfies
\[
 \exp\left(
 a_it-\lambda_r\sum_{\ell=1}^t(X_{i,\ell}-\mu_i)
 \right)
 <\frac1{\eta q_r}.
\]
Consequently, for every integer $t\ge2b_r/a_i$,
\[
 \begin{aligned}
 \Pp(U_i>t)
 &\le
 \Pp\left(
 \lambda_r\sum_{\ell=1}^t(X_{i,\ell}-\mu_i)>a_it-b_r
 \right)\\
 &\le
 \exp\left(-\frac{(a_it-b_r)^2}{2\lambda_r^2t}\right)\\
 &\le
 \exp\left(-\frac{a_i^2t}{8\lambda_r^2}\right)
 \le
 \exp\left(-\frac{\Delta_i^2t}{128}\right).
 \end{aligned}
\]
Here the penultimate inequality uses $a_it-b_r\ge a_it/2$, and the
last uses $a_i/\lambda_r\ge\Delta_i/4$.

Let $t_0=\lceil2b_r/a_i\rceil$. Summing the tail probabilities gives
\[
 \E U_i
 \le t_0+\sum_{t=t_0}^{\infty}
             e^{-\Delta_i^2t/128}
 \le130\Delta_i^{-2}(b_r+1).
\]
Indeed, $a_i\ge\Delta_i^2/16$ gives
$t_0\le32b_r\Delta_i^{-2}+1$, while $\Delta_i\le1$ gives
\[
 \sum_{t=t_0}^{\infty}e^{-\Delta_i^2t/128}
 \le\frac1{1-e^{-\Delta_i^2/128}}
 \le129\Delta_i^{-2}.
\]

Summing over the suboptimal arms and using
$b_r=\log(1/\eta)+\log(1/q_r)$ yields
\[
 \begin{aligned}
 \sum_{i\ne *}\E U_i
 &\le130\sum_{r\in\mathcal U}
 H_r\left[\log(1/\eta)+\log(1/q_r)+1\right]\\
 &\le520\sum_{r\in\mathcal U}
 w_r\left[\log(1/\eta)+\log(1/q_r)+1\right]\\
 &=520W\left[\log(1/\eta)+\Ent(q)+1\right].
 \end{aligned}
\]
The second inequality uses $H_r\le4w_r$, while the last equality uses
$q_r=w_r/W$.

It remains to account for samples from the optimal arm and for the
confirmation step. During the elimination phase, suboptimal arm $i$
receives at most $U_i$ samples. If the optimal arm receives $t$
elimination samples, then in its final sweep at least one suboptimal
arm is also active and receives its $t$th sample. Hence $U_i\ge t$ for
that arm, so the optimal arm receives at most
$\max_{i\ne *}U_i$ samples before the elimination phase ends. This
remains true if the optimal arm is itself eliminated. Hence the
elimination cost is bounded pathwise by
\[
 \sum_{i\ne *}U_i+\max_{i\ne *}U_i
 \le2\sum_{i\ne *}U_i.
\]
The confirmation step uses at most $n_\eta$ further samples. Since
$h^{-2}\le W$, the expected cost of the uncapped trial is therefore at
most
\[
 2000W\left[\log(1/\eta)+\Ent(q)+1\right].
\]
Comparing this with \eqref{eq:lower-cap}, Markov's inequality shows
that the probability of reaching the deterministic sample limit is at
most $1/8$.

Finally, on the target instance,
\[
 m-\tau>\frac g3>\frac h6.
\]
Thus the confirmation threshold $\tau+h/12$ lies at least $h/12$
below $m$, and \eqref{eq:subgaussian-tail} gives
\[
 \Pp\left(
 \overline X_* \le \tau+\frac h{12}
 \right)
 \le
 e^{-n_\eta h^2/288}
 \le\eta.
\]
The bounds on $\E U_i$ imply that $U_i<\infty$ almost surely
for every suboptimal arm $i$. If the optimal arm is never
eliminated, the uncapped elimination phase therefore ends
with that arm as the sole survivor. Therefore,
unless the optimal arm is eliminated, the sample limit is reached, or
the confirmation fails, the trial returns the optimal arm. Its success
probability is at least
\[
 1-\eta-\frac18-\eta
 =1-\frac18-2\eta>\frac34.
\]
All estimates are uniform over relabelings of the target instance.

\subsection{Repetition and almost-sure stopping}\label{s:rass}

Run independent trials with error levels
\[
 \eta_j=\delta2^{-j-2},\qquad j\ge0,
\]
and stop at the first trial that returns an arm. On every input instance, the probability that some
trial returns an incorrect arm is at most
\[
 \sum_{j\ge0}\eta_j=\frac{\delta}{2}.
\]
For any relabeling of the target instance, each trial returns the optimal arm with probability at
least \(3/4\). Independence therefore implies that the probability of reaching trial \(j\) is at
most \(4^{-j}\). Moreover,
\[
 B_{\eta_j}
 \le C W\bigl[\log(1/\delta)+\Ent(q)+j+1\bigr].
\]
Consequently, the expected number of samples used by the sequence of trials is at most
\[
 \sum_{j\ge0}4^{-j}B_{\eta_j}
 \le C W\bigl[\log(1/\delta)+\Ent(q)+1\bigr].
\]

The sequence of trials need not stop almost surely when its fixed parameters do not describe the
input instance. To ensure almost-sure stopping on every instance, run it in parallel with an
independent auxiliary procedure that is \(\delta/2\)-correct on \(\mathcal S_K\). Alternate between
the next sample requested by the trial sequence and the next sample requested by the auxiliary
procedure, and return the first arm produced by either procedure. Before the trial sequence returns
an arm, the combined algorithm uses at most twice as many samples, up to one additional sample.
Thus its expected cost on any relabeling of the target instance remains bounded by
\[
 C W\bigl[\log(1/\delta)+\Ent(q)+1\bigr].
\]

We construct the auxiliary procedure using simultaneous confidence intervals. Put
\(\beta=\delta/2\). At stage \(k\ge0\), draw \(2^k\) fresh samples from each arm and let
\(\overline X_{i,k}\) be their average. Form the intervals
\[
 \left[\overline X_{i,k}-c_k,\overline X_{i,k}+c_k\right],
 \qquad
 c_k=\sqrt{2^{1-k}\log\frac{8K(k+1)^2}{\beta}}.
\]
Stop and return arm \(i\) if its lower endpoint exceeds the upper endpoint of every other arm.
For each arm and stage, the sub-Gaussian tail bound gives
\[
 \Pp\bigl(|\overline X_{i,k}-\mu_i|>c_k\bigr)
 \le \frac{\beta}{4K(k+1)^2}.
\]
A union bound over all arms and stages shows that the probability of any confidence-interval
failure is at most
\[
 \frac{\beta}{4}\sum_{k\ge0}\frac1{(k+1)^2}
 =\frac{\beta\pi^2}{24}<\beta.
\]
Whenever all intervals contain their respective means, any arm returned by the procedure is
optimal. The auxiliary procedure is therefore \(\beta\)-correct.

It remains to verify almost-sure stopping. For every fixed \(e>0\),
\[
 \Pp\bigl(|\overline X_{i,k}-\mu_i|>e\bigr)
 \le2\exp(-2^{k-1}e^2),
\]
and the right side is summable in \(k\). Applying Borel--Cantelli for each arm and for
\(e=1,\frac12,\frac13,\ldots\) gives
\[
 \overline X_{i,k}\longrightarrow\mu_i
 \qquad\text{almost surely for every arm }i.
\]
Since \(c_k\to0\) and the optimal arm is unique, its lower endpoint eventually exceeds every
suboptimal arm's upper endpoint. Hence the auxiliary procedure, and therefore the combined
algorithm, stops almost surely on every instance in \(\mathcal S_K\).

The two component procedures have error probabilities at most \(\delta/2\), so the combined
algorithm is \(\delta\)-correct. Finally, \eqref{eq:lower-upperweights} and
\eqref{eq:lower-entcomp} give
\[
 W\le H,
 \qquad
 W\Ent(q)\le H\Ent(I).
\]
Since \(\log(1/\delta)>\log 100\), the additive \(W\) term can be absorbed into
\(H\log(1/\delta)\). Thus, uniformly over all relabelings of the target instance,
\[
 \E_I T
 \le C H\bigl[\log(1/\delta)+\Ent(I)\bigr],
\]
which proves \zcref{prop:subgaussian-instance-upper}. 

\section{The uniform algorithm}\label{sec:algorithm}
We construct the algorithm in \zcref{thm:uniform} for $\mathcal S_K$ by modifying the
Entropy-Elimination and Complexity-Guessing procedures of Chen, Li and Qiao
\cite[algorithms 4--5]{CLQ17}. We retain their geometrically increasing stage budgets, two cumulative budget
tests, and elimination thresholds. We adjust the maximum number of rounds and fraction-threshold increments
to account for all sampling costs, and impose a deterministic sample limit at each stage. The
analysis uses the comparison process in \zcref{sec:uniform-analysis}.
Each stage starts with all $K$ arms. One cumulative sum limits the total elimination error budget,
and the other limits its expected sampling cost.

\subsection{Sampling subroutines}

By the reduction in \zcref{sec:instance-upper}, it suffices to consider
\(0<\delta<0.01\). Each subroutine call uses samples independent of the history before the call.
Consequently, the guarantees below hold conditionally on that history and on the possibly random
inputs to the call. They apply to \(1\)-sub-Gaussian reward laws and use only the sample-mean bound
\eqref{eq:subgaussian-tail} and independence between sample blocks.

We refer to Mean and Fraction calls made directly by the main algorithm as \emph{outer tests}, to
distinguish them from calls made internally by other subroutines.

Median and Fraction require nonempty finite input sets, whereas Eliminate also accepts the empty
set. Throughout, accuracy parameters and differences between mean thresholds lie in \((0,1]\),
error levels lie in \((0,0.1)\), and differences between fraction thresholds lie in \((0,0.1]\).
These ranges cover every call made by the algorithm. The implementations and proofs are given in
\zcref{app:primitives}. Here and below, \(C\) denotes an absolute constant.

\begin{enumerate}[label=(\roman*),leftmargin=*]
\item \emph{Mean estimation.}
The call \(\Mean(i,e,\alpha)\) uses
\(O(e^{-2}\log(1/\alpha))\) samples deterministically and returns an estimate of \(\mu_i\) whose
error is at most \(e\), except with probability \(\alpha\).

\item \emph{Median elimination.}
The call \(\Median(S,e)\) uses \(O(|S|e^{-2})\) samples deterministically and, with conditional
probability at least \(0.99\), returns an arm \(i\in S\) satisfying
\[
 \mu_i\ge\max_{j\in S}\mu_j-e.
\]

\item \emph{Fraction testing.}
For \(l<u\) and \(\theta_-<\theta_+\), the call
\(\Fraction(S,l,u,\theta_-,\theta_+,\alpha)\) has deterministic cost at most
\[
 C(u-l)^{-2}\log(1/\alpha)
   (\theta_+-\theta_-)^{-2}
   \log\frac{1}{\theta_+-\theta_-}.
\]
Except with probability \(\alpha\), its output obeys the following implications:
\[
 \begin{aligned}
  \mathrm{True}
  &\implies |\{i\in S:\mu_i<u\}|>\theta_-|S|,\\
  \mathrm{False}
  &\implies |\{i\in S:\mu_i<l\}|<\theta_+|S|.
 \end{aligned}
\]
Fraction is not required to estimate either proportion. It need only return an answer whose
corresponding implication is valid; on some inputs, both answers are valid.

\item \emph{Elimination.}
The call \(\Eliminate(S,l,u,\alpha)\) returns a subset \(S'\subseteq S\) and terminates almost
surely. If \(S=\varnothing\), it returns immediately. Its conditional expected cost is at most
\[
 C|S|(u-l)^{-2}\log(1/\alpha).
\]
If \(S\ne\varnothing\), fix, before the call, a designated arm
\(i^\dagger\in S\) attaining \(\max_{i\in S}\mu_i\). Except with
probability \(\alpha\), the returned set simultaneously satisfies
\[
 |\{i\in S':\mu_i<l\}|\le0.1|S'|,
\]
and, if \(\mu_{i^\dagger}\ge u\), it contains \(i^\dagger\). When
\(S=\varnothing\), the displayed inequality holds deterministically. 
In addition, for any \(k\) arms designated before the call whose means are all at least \(u\), the
probability that all \(k\) are removed is at most \(\alpha^k\).
\end{enumerate}
We use a product bound to control the joint loss of several good arms across stages. Conditional
on the inputs and the history before an Eliminate call, generate independent estimation blocks in
advance for every arm and every internal stage. If an arm has mean at least \(u\), then it can be
removed only if one of its own estimates is inaccurate. The error budgets for that arm sum to at
most \(\alpha/10\), so this event has probability at most \(\alpha/10\). Because the estimation
blocks are independent across arms, the probability that \(k\) fixed arms of mean at least \(u\)
are all removed is at most
\[
 \left(\frac{\alpha}{10}\right)^k\le\alpha^k.
\]
This conclusion does not require the Fraction calls to be correct: an incorrect Fraction answer
may cause an unnecessary deletion step, but no such step removes an arm of mean at least \(u\)
when that arm's estimate is accurate. This is the argument of \cite[lemma B.9]{CLQ17}. The full
construction, including its behavior on the empty set, is given in \zcref{app:elimination}.

\subsection{Stages and stopping rules}

Let \(J\) and \(M\) be the absolute constants specified in
\zcref{lem:uniform-aux,lem:uniform-success}, respectively.
After fixing them, choose a sufficiently small absolute constant \(c>0\) and set
\[
 \eta=c\delta,\qquad \rho=\eta^2,\qquad a=\log(1/\eta),\qquad
 B_t=100^t,\qquad L_t=a+\log(t+1),\qquad t=1,2,\ldots.
\]
Each stage starts with all \(K\) arms and uses samples and internal randomness independent of every
other stage. We say that a stage \emph{accepts} if it returns an arm and \emph{rejects} if it ends
without doing so.

For stage \(t\), set
\[
 R_t=\max\{0,\lfloor\log_4(B_t/L_t)\rfloor-J\}.
\]
If \(R_t=0\), reject immediately. Otherwise, write \(B=B_t\) and \(R=R_t\), and initialize
\[
 S_1=[K],\qquad U_1=V_1=0.
\]
The variables \(U_r\) and \(V_r\) are running totals used by the two budget checks. At the start of round \(r\),
\[
 U_r=4\sum_{\substack{k<r\\\text{\rm Eliminate was called in round }k}}w_k,
 \qquad
 V_r=\sum_{\substack{k<r\\\text{\rm round }k\text{\rm\ passed both budget checks}}}z_k.
\]
They do not record the number of samples used, since that quantity is controlled by the separate
deterministic limit involving \(M\).

Define the fraction thresholds by
\[
 \theta_0=0.3,\qquad
 \theta_r=\theta_{r-1}+\frac{1}{10(R+1-r)^2},
 \qquad 1\le r\le R.
\]
They are increasing and satisfy
\[
 \theta_R
 =0.3+\frac1{10}\sum_{k=1}^R\frac1{k^2}
 <0.3+\frac{\pi^2}{60}<0.5.
\]
Thus \(\theta_r\in[0.3,0.5)\), and \(v<r\) implies
\(\theta_v\le\theta_{r-1}\).

At the beginning of round \(r\), return the unique arm in \(S_r\) if \(|S_r|=1\). Otherwise, reject
if \(S_r=\varnothing\) or \(r>R\). If the stage continues, set
\[
 \varepsilon_r=2^{-r},\qquad w_r=|S_r|4^r.
\]
The first budget check rejects if
\[
 U_r+4w_r\ge B.
\]
If it passes, then \(w_r<B/4\), and we may define the positive quantity
\[
 z_r=w_r\log\frac{B}{\eta w_r}.
\]
The second budget check rejects if
\[
 V_r+z_r\ge100B.
\]
If both checks pass, perform the following operations:
\begin{enumerate}[leftmargin=*]
\item Update the second budget by setting
\[
 V_{r+1}=V_r+z_r,
 \qquad
 \alpha_{t,r}=\frac{\eta}{50t^2r^2}.
\]

\item Call
\[
 \widehat i_r=\Median(S_r,\varepsilon_r/8),
 \qquad
 \widehat\mu_r=\Mean(\widehat i_r,\varepsilon_r/8,\alpha_{t,r}).
\]

\item Let \(F_r\) be the output of
\[
 \Fraction\left(
 S_r,\,
 \widehat\mu_r-\tfrac74\varepsilon_r,\,
 \widehat\mu_r-\tfrac98\varepsilon_r,\,
 \theta_{r-1},\theta_r,\alpha_{t,r}
 \right).
\]

\item If \(F_r=\mathrm{True}\), set
\[
 \beta_r=\frac{4w_r\rho}{B}
\]
and call
\[
 S_{r+1}
 =\Eliminate\left(
 S_r,\,
 \widehat\mu_r-\tfrac34\varepsilon_r,\,
 \widehat\mu_r-\tfrac58\varepsilon_r,\,
 \beta_r
 \right).
\]
Then set \(U_{r+1}=U_r+4w_r\). If \(F_r=\mathrm{False}\), make no elimination call and set
\[
 S_{r+1}=S_r,\qquad U_{r+1}=U_r.
\]
\end{enumerate}

Maintain a counter for the samples used within the stage. Before taking any sample, including one
requested by a subroutine, reject the stage if that sample would make the counter exceed
\(\lceil MB\rceil\). After a rejection, begin stage \(t+1\) with fresh samples and randomness. If
the stage returns an arm, terminate the algorithm and output that arm.

The singleton check precedes the condition \(r>R\). Consequently, if round \(R\) produces a
singleton \(S_{R+1}\), the algorithm returns that arm at the beginning of round \(R+1\); otherwise,
it rejects at that point.

\section{Progress and error control}\label{sec:uniform-analysis}

Fix an instance \(I\in\mathcal S_K\), with optimal mean \(m\). We analyze the algorithm through a
comparison process in which the outer Mean and Fraction outputs are made valid by construction.
Median and Eliminate are not corrected. Successful Median calls will control the sizes of the
active sets, while invalid Eliminate calls will be handled separately. In particular, we never
condition on all future calls being valid.

\subsection{A comparison process}

For an outer Mean call on arm \(i\) with accuracy \(e\), replace its output by its projection onto
\[
 [\mu_i-e,\mu_i+e].
\]
Thus the corrected estimate always satisfies the advertised accuracy guarantee and agrees with the
original output whenever that output is accurate.

For an outer Fraction call with parameters \(S,l,u,\theta_-,\theta_+\), call True valid if
\[
 |\{i\in S:\mu_i<u\}|>\theta_-|S|
\]
and False valid if
\[
 |\{i\in S:\mu_i<l\}|<\theta_+|S|.
\]
At least one answer is valid. Indeed, if True is invalid, then
\[
 |\{i\in S:\mu_i<l\}|
 \le |\{i\in S:\mu_i<u\}|
 \le\theta_-|S|<\theta_+|S|,
\]
so False is valid. Replace an invalid Fraction output by the other answer, leaving a valid output
unchanged.

These corrections are measurable functions of the current inputs, outputs, and true means, and
do not use future samples. They are introduced only for the analysis. They do not alter the samples
already drawn or the cost of the corrected call, and the outputs of Median and Eliminate are never
corrected.

Couple the actual and comparison processes using the same randomness until their first
disagreement. Such a disagreement can occur only when an outer Mean or Fraction output is invalid.
There are at most two such calls in each round, each with conditional failure probability
\(\alpha_{t,r}=\eta/(50t^2r^2)\). Hence
\begin{equation}\label{eq:uniform-coupling}
 \Pp(\text{the two processes ever disagree})
 \le2\sum_{t,r\ge1}\frac{\eta}{50t^2r^2}
 \le C\eta.
\end{equation}
For a fixed stage \(t\), the same argument gives a bound of \(C\eta/t^2\). We will not condition on
agreement of the two processes. Because every Median call uses fresh samples, its conditional
failure probability in the comparison process remains at most
\[
 \kappa=0.01.
\]

For a call \(\Eliminate(S,l,u,\alpha)\), designate before the call one arm attaining
\(\max_{i\in S}\mu_i\). Call the returned set \(S'\) \emph{valid} if
\[
 |\{i\in S':\mu_i<l\}|\le0.1|S'|
\]
and, whenever the designated arm has mean at least \(u\), that arm belongs to \(S'\).

We first record why the true optimal arm remains active before the first invalid Eliminate call.
The corrected estimate in round \(r\) satisfies
\[
 \widehat\mu_r
 \le\mu_{\widehat i_r}+\frac{\varepsilon_r}{8}
 \le m+\frac{\varepsilon_r}{8}.
\]
Consequently, the upper thresholds used by Fraction and Eliminate satisfy
\[
 \widehat\mu_r-\tfrac98\varepsilon_r\le m-\varepsilon_r,
 \qquad
 \widehat\mu_r-\tfrac58\varepsilon_r\le m-\tfrac12\varepsilon_r.
\]
Thus, whenever Eliminate is called while the optimal arm is active, its upper threshold is below
\(m\). A valid call therefore retains the optimal arm. Induction over the rounds proves the claim.

Call the Median output in round \(r\) \emph{successful} if
\[
 \mu_{\widehat i_r}
 \ge\max_{i\in S_r}\mu_i-\frac{\varepsilon_r}{8}.
\]
Before the first invalid Eliminate call, the optimal arm lies in \(S_r\), so the maximum on the
right is \(m\). If Median is successful, then the corrected Mean output satisfies
\[
 \widehat\mu_r\ge m-\frac{\varepsilon_r}{4}.
\]
It follows that the lower thresholds used by Fraction and Eliminate are bounded below by
\[
 \widehat\mu_r-\tfrac74\varepsilon_r\ge m-2\varepsilon_r,
 \qquad
 \widehat\mu_r-\tfrac34\varepsilon_r\ge m-\varepsilon_r.
\]

Suppose that Median is successful and that every Eliminate call through
the current round is valid. If Fraction returns False,
then
\[
 |\{i\in S_r:\mu_i<
       \widehat\mu_r-\tfrac74\varepsilon_r\}|
 <\theta_r|S_r|<\frac{|S_r|}{2}.
\]
More than half of \(S_r\) therefore consists of the optimal arm and suboptimal arms with
\(\Delta_i\le2\varepsilon_r\). Since a False answer leaves the active set unchanged,
\[
 |S_{r+1}|
 \le2\bigl(1+|\{i\ne *:\Delta_i\le2\varepsilon_r\}|\bigr).
\]
If Fraction returns True, validity of Eliminate implies that at least \(0.9|S_{r+1}|\) returned
arms have mean at least \(m-\varepsilon_r\). All such arms are included among the optimal arm and
the suboptimal arms with \(\Delta_i\le\varepsilon_r\). Hence the same, slightly weaker bound holds:
\begin{equation}\label{eq:uniform-progress}
 |S_{r+1}|
 \le2\bigl(1+|\{i\ne *:\Delta_i\le2\varepsilon_r\}|\bigr).
\end{equation}

There is a sharper conclusion once the round scale is below the smallest gap. Let
\[
 Q=\lfloor\log_4D\rfloor,\qquad D=g^{-2}.
\]
Then
\[
 2^{-(Q+1)}<g\le2^{-Q}.
\]
For \(r\ge Q+2\),
\[
 2\varepsilon_r=2^{1-r}\le2^{-(Q+1)}<g.
\]
Thus every suboptimal arm lies below both lower thresholds displayed above. If \(|S_r|\ge2\), at
least half of \(S_r\) then lies below the Fraction lower threshold, so False is not a valid answer
and the comparison process must return True. A valid Eliminate call retains the optimal arm. If
\(n\) suboptimal arms remain afterward, its validity also gives
\[
 n\le0.1(n+1),
\]
which forces \(n=0\). Therefore every successful Median round \(r\ge Q+2\), before any invalid
Eliminate call, leaves \(S_{r+1}\) equal to the singleton containing the optimal arm.

\subsection{Active-set bounds}

We now use \eqref{eq:uniform-progress} to control the two cumulative quantities appearing in the
budget tests: the linear sum of the round weights and its logarithmically weighted analogue. The
next lemma is a stopped-process version of \cite[lemmas B.4--B.5]{CLQ17}.

Recall \(\Phi\) from \eqref{eq:entropy-functional}. For a summable nonnegative sequence
\(v=(v_r)\) with \(V=\sum_rv_r\), we use the extension
\[
 \Phi(v)=\sum_{r:v_r>0}v_r\log\frac{V}{v_r},
\]
with \(\Phi(0)=0\).

We first remove the deterministic sample limit from the stage. The resulting bounds continue to
hold when the limit is restored, because the limit can only terminate the stage earlier. A round is
\emph{entered} when the first budget test is reached and \emph{executed} when both budget tests are
passed. We adopt the convention
\[
 c_r\log\left(e+\frac{B}{\eta c_r}\right)=0
 \qquad\text{when }c_r=0.
\]

\begin{lemma}\label{lem:uniform-charging}
Consider one stage of the comparison process without its deterministic sample limit. Stop the
analysis at the first invalid Eliminate call, at the first budget rejection, or after the last
permitted round, whichever occurs first. The round containing an invalid Eliminate call and a round
rejected by a budget test are both included.

For every entered round with at least two active arms, set
\[
 c_r=4^r|S_r|,
\]
and set \(c_r=0\) for all subsequent rounds. Thus \(c_r=w_r\) on executed rounds; for a round rejected immediately by a budget test, \(c_r\) records the corresponding round weight. There is a deterministic nonnegative
sequence \(b=(b_r)_{r\ge1}\) such that
\[
 \E c_r\le b_r,\qquad
 \sum_r b_r\le CH,\qquad
 \Phi(b)\le CH(\Ent(I)+1).
\]
Consequently, whenever \(B\ge H\),
\begin{equation}\label{eq:uniform-charges}
 \E\sum_r c_r\le CH,
 \qquad
 \E\sum_r c_r\log\left(e+\frac{B}{\eta c_r}\right)
 \le CH\bigl[a+\Ent(I)+1+\log(B/H)\bigr].
\end{equation}

We also have the following localized version. Fix \(2\le s\le K\), an integer \(q\ge0\), and a deterministic
cutoff \(r_0\in\mathbb N\cup\{\infty\}\). Designate \(s-1\) arms, including the optimal arm, and
suppose that each designated arm has mean at least the upper threshold of every Eliminate call made
in a round \(r<r_0\). Assume that every nondesignated arm belongs to a gap group with index at most
\(q\), and set
\[
 H_{\rm tail}
 =\sum_{\substack{i\ne *\\i\text{ not designated}}}\Delta_i^{-2}.
\]
For this localized estimate only, assign all \(s-1\) designated arms to group \(q\),
regardless of their actual gaps.

Let \(\mathcal R\) contain the entered rounds \(r<r_0\), up to and including the first round
\(r\ge q+2\) in which Median succeeds. If no such round occurs, let \(\mathcal R\) contain every
entered round before the cutoff or an earlier termination. A round rejected immediately by a budget
test is included in \(\mathcal R\). Then
\[
 \E\sum_{r\in\mathcal R}c_r
 \le C\bigl[H_{\rm tail}+(s-1)4^q\bigr].
\]
\end{lemma}
\begin{proof}
\medskip
\noindent\emph{Step 1: A deterministic bound for the round weights.}
For the global estimate, treat the optimal arm as if it had gap \(g\), and hence as if it belonged
to group \(Q=\lfloor\log_4D\rfloor\). Define the augmented group sizes and weights by
\[
 n'_j=|G_j|+\mathbf 1_{\{j=Q\}},
 \qquad
 v_j=n'_j4^j,
 \qquad
 A_\ell=\sum_{j\ge\ell}n'_j.
\]
Thus \(\sum_jn'_j=K\).

Suppose that Median succeeds in round \(k\). If \(k\le Q+1\), then
\eqref{eq:uniform-progress} implies that any subsequently entered round has at most
\(2A_{k-1}\) active arms. If \(k\ge Q+2\), the successful round leaves a singleton, so no later
round is entered.

To bound \(\E c_r\), decompose according to the last successful Median call before round \(r\).
The probability that all preceding \(r-1\) Median calls fail is at most \(\kappa^{r-1}\).
If the last success occurs in round \(k<r\), the subsequent \(r-k-1\) Median calls must all fail,
an event of conditional probability at most \(\kappa^{r-k-1}\). Therefore
\[
 \E c_r
 \le4^r\left[
 K\kappa^{r-1}
 +2\sum_{k=1}^{r-1}\kappa^{r-k-1}A_{k-1}
 \right].
\]
An earlier invalid Eliminate call, budget rejection, or other termination can only set some later
\(c_r\)'s to zero, so the same bound holds for the stopped process.

For a fixed augmented arm in group \(j\), its contribution to the expression in brackets is at
most
\[
 \kappa^{r-1}
 +2\sum_{k=1}^{\min\{r-1,j+1\}}\kappa^{r-k-1}
 \le4\kappa^{\max\{0,r-j-2\}}.
\]
Define the two-sided kernel
\[
 k_d=
 \begin{cases}
  4^d, & d\le2,\\
  4^d\kappa^{d-2}, & d>2.
 \end{cases}
\]
Then
\[
 \E c_r\le b_r,
 \qquad
 b_r=4\sum_{j\ge0}v_jk_{r-j},
 \qquad r\ge1.
\]
Because \(\kappa=0.01\),
\[
 \sum_{d\in\mathbb Z}k_d
 =\sum_{d\le2}4^d+\sum_{d\ge3}4^d\kappa^{d-2}
 =\frac{64}{3}+\frac{64\kappa}{1-4\kappa}
 =22.
\]
In particular,
\[
 \sum_{r\ge1}b_r\le88\sum_jv_j.
\]

\medskip
\noindent\emph{Step 2: Entropy calculation.}
The probability vector \((k_d/22)_{d\in\mathbb Z}\) has finite entropy because both tails decay
geometrically. Put
\[
 V=\sum_jv_j,
\]
let \(X\) have law \((v_j/V)_j\), and let \(Y\) be independent of \(X\) with law
\((k_d/22)_d\). Extend the convolution to all \(r\in\mathbb Z\) by setting
\[
 \widetilde b_r=4\sum_jv_jk_{r-j}.
\]
Its total mass is \(88V\), and its normalized distribution is the law of \(X+Y\). Since \(X+Y\)
is a function of \((X,Y)\),
\[
 \Ent(X+Y)\le\Ent(X,Y)=\Ent(X)+\Ent(Y).
\]
Consequently,
\[
 \Phi(\widetilde b)
 =88V\Ent(X+Y)
 \le88\Phi(v)+88V\Ent(Y).
\]
The sequence \(b\) is obtained by restricting \(\widetilde b\) to \(r\ge1\). Coordinatewise
monotonicity of \(\Phi\), first for finite truncations and then by passage to the limit, gives
\[
 \Phi(b)\le88\Phi(v)+88V\Ent(Y).
\]

The closest competitor belongs to \(G_Q\), so \(|G_Q|\ge1\). Adding the artificial optimal arm
therefore at most doubles the weight of that group. Since
\[
 |G_j|4^j\le H_j,
\]
we have \(v_j\le2H_j\) for every \(j\). The monotonicity and homogeneity of \(\Phi\) now give
\[
 V\le2H,
 \qquad
 \Phi(v)\le\Phi((2H_j)_j)=2H\Ent(I).
\]
It follows that
\[
 \sum_r b_r\le CH,
 \qquad
 \Phi(b)\le CH(\Ent(I)+1),
\]
which proves the first set of assertions.

\medskip
\noindent\emph{Step 3: Logarithmically weighted cost.}
Set
\[
 f(z)=z\log\left(e+\frac{B}{\eta z}\right),
 \qquad f(0)=0.
\]
This function is increasing and concave. Hence Jensen's inequality, followed by
\(\E c_r\le b_r\), yields
\[
 \E f(c_r)\le f(\E c_r)\le f(b_r).
\]
Let \(N=\sum_rb_r\). For every \(r\) with \(b_r>0\),
\[
 e+\frac{B}{\eta b_r}
 \le\frac{N}{b_r}
       \left(e+\frac{B}{\eta N}\right).
\]
After taking logarithms, multiplying by \(b_r\), and summing,
\[
 \sum_rf(b_r)
 \le N\log\left(e+\frac{B}{\eta N}\right)+\Phi(b).
\]
Since \(N\le CH\), \(B\ge H\), and \(a=\log(1/\eta)\), monotonicity of \(f\) gives
\[
 N\log\left(e+\frac{B}{\eta N}\right)
 \le CH\bigl[a+1+\log(B/H)\bigr].
\]
Together with the entropy bound for \(b\), this proves \eqref{eq:uniform-charges}.

\medskip
\noindent\emph{Step 4: Localized estimate.}
Let \(\mathcal D\) be the set of \(s-1\) designated arms and put
\[
 W_{\mathrm{loc}}=H_{\mathrm{tail}}+(s-1)4^q.
\]
For this step, set \(c_r=0\) outside the rounds in \(\mathcal R\). Thus the cutoff \(r_0\), any
earlier termination, and the first successful Median call at or after round \(q+2\) are all
incorporated into the definition of \(c_r\).

Assign the designated arms artificially to group \(q\) and define
\[
 n_j^{\mathrm{loc}}
 =|\{i\in G_j:i\notin\mathcal D\}|
   +(s-1)\mathbf 1_{\{j=q\}},
 \qquad
 A_\ell^{\mathrm{loc}}
 =\sum_{j\ge\ell}n_j^{\mathrm{loc}}.
\]
By the assumptions on the nondesignated arms,
\[
 \sum_{j=0}^q n_j^{\mathrm{loc}}=K,
 \qquad
 \sum_{j=0}^q n_j^{\mathrm{loc}}4^j
 \le H_{\mathrm{tail}}+(s-1)4^q
 =W_{\mathrm{loc}}.
\]

For \(k\le q+1\), every designated arm is counted in
\(A_{k-1}^{\mathrm{loc}}\). Moreover, every nondesignated arm with
\(\Delta_i\le2\varepsilon_k=2^{1-k}\) belongs to a group of index at least \(k-1\). Therefore, if a
successful Median round \(k\le q+1\) is followed by another entered round,
\eqref{eq:uniform-progress} gives
\[
 |S_{k+1}|\le2A_{k-1}^{\mathrm{loc}}.
\]
The same decomposition according to the last successful Median call now yields, for
\(1\le r\le q+2\),
\[
 \E c_r
 \le4^r\left[
 K\kappa^{r-1}
 +2\sum_{k=1}^{r-1}
     \kappa^{r-k-1}A_{k-1}^{\mathrm{loc}}
 \right].
\]
Applying the kernel estimate from Step 1 and summing over these rounds gives
\[
 \E\sum_{r=1}^{q+2}c_r
 \le4\sum_{j=0}^q n_j^{\mathrm{loc}}4^j
          \sum_{r=1}^{q+2}k_{r-j}
 \le88W_{\mathrm{loc}}.
\]

It remains to control rounds after \(q+2\). The active set never grows. For $j \ge 1$, in order for 
\(c_{q+2+j}\) to contribute to the localized sum, the Median calls in rounds
\(q+2,\ldots,q+1+j\) must all fail; otherwise the sum stops after the first successful call.
Conditionally on the history upon entering round \(q+2\), these \(j\) failures have probability at
most \(\kappa^j\). Hence
\[
 \E c_{q+2+j}
 \le(4\kappa)^j\E c_{q+2}
 \le C W_{\mathrm{loc}}(4\kappa)^j,
 \qquad j\ge 1.
\]
Since \(4\kappa<1\), summing this geometric tail and combining it with the estimate for the earlier
rounds gives
\[
 \E\sum_{r\in\mathcal R}c_r
 \le C W_{\mathrm{loc}}
 =C\bigl[H_{\mathrm{tail}}+(s-1)4^q\bigr].
\]
\end{proof}

\subsection{Joint and localized errors}

Until the end of this section, all probabilities refer to the comparison process. We first control
the joint probability that several arms are eliminated before the rounds at which their gaps become
distinguishable.

Fix one stage, and let \(\mathcal E\) be the set of rounds in which Eliminate is called. Whenever
such a call is made in round \(r\), the first budget test has ensured
\(U_r+4w_r<B\), and the subsequent update is \(U_{r+1}=U_r+4w_r\). Hence, pathwise,
\begin{equation}\label{eq:uniform-riskbudget}
 \sum_{r\in\mathcal E}\beta_r
 =\frac{\rho}{B}\sum_{r\in\mathcal E}4w_r
 =\frac{\rho}{B}\sum_{r\in\mathcal E}4|S_r|4^r
 <\rho.
\end{equation}

For a suboptimal arm \(i\in G_q\), define
\[
 M_i
 =\mathbf 1\{i\text{ is removed by an Eliminate call in some round }r<q\}.
\]
Thus \(M_i=1\) means that \(i\) is removed before the algorithm reaches the scale associated with
its gap. For the optimal arm, define
\[
 M_*
 =\mathbf 1\{*\text{ is removed by an Eliminate call in any round}\}.
\]
We refer to either event as a \emph{premature elimination}, corresponding to \cite[definition B.8]{CLQ17}.

The next lemma combines \eqref{eq:uniform-riskbudget} with the product guarantee
for Eliminate. Its backward-induction argument adapts \cite[lemma B.9]{CLQ17} to the comparison
process.

\begin{lemma}\label{lem:uniform-error}
For a fixed collection of $k$ arms in one stage,
\begin{equation}\label{eq:uniform-joint}
 \Pp(M_i=1\text{ for every arm in the collection})\le\rho^k.
\end{equation}
\end{lemma}
\begin{proof}
The corrected mean estimate is at most $m+\varepsilon_r/8$, even if the true optimal arm was removed
earlier. Thus the upper elimination threshold in round $r$ is at most $m-\varepsilon_r/2$. For an arm in gap group $q>r$, $\Delta_i\le2^{-q}\le\varepsilon_r/2$, so its mean is at least
the upper elimination threshold. The optimal arm also lies above this threshold. Hence every arm
whose removal in round $r$ would be premature satisfies the hypothesis of the simultaneous-removal
bound.
Conditional on the past, a specified set of $j$ eligible arms is therefore eliminated in that call
with probability at most $\beta_r^j$.

Induct backward over the finitely many remaining rounds. If a specified arm has already been
removed at or after its gap-group index, the joint event is impossible. Immediately before a call
in round $r$, let
\[
 u=\rho-\sum_{\substack{v\in\mathcal E\\v<r}}\beta_v.
\]
This quantity is determined by the past. By \eqref{eq:uniform-riskbudget}, the current and future
calls together receive total error allowance at most $u$ on every continuation. 
Suppose that $k$ specified arms still have to be eliminated
prematurely. If the current round has reached the gap-group index of
any specified suboptimal arm, the joint event is impossible.
Otherwise all $k$ specified arms are eligible for premature
elimination in the current call. The induction claim is that their
joint probability is at most $u^k$.  
Condition on the history immediately before the next Eliminate call,
including the preceding Mean and Fraction outputs. Its error budget $\beta$ is then fixed, and
the total budget available to later calls is at most $u-\beta$ on every continuation. The conditional probability of any particular exact subset of $j$ premature removals
is at most $\beta^j$, since it is bounded by the probability that those $j$ arms are all removed. Conditional on the resulting history, the probability that the remaining $k-j$ arms are all
eliminated prematurely is at most $(u-\beta)^{k-j}$. Summing over
subsets gives
\[
 \sum_{j=0}^k \binom{k}{j}\beta^j(u-\beta)^{k-j}=u^k.
\]
If no further call is made and $k>0$ specified arms still require premature elimination, the
conditional probability is zero. If $k=0$, the event is already satisfied and its conditional
probability is one. These are the terminal cases of the induction. Initially the available budget
is $\rho$, so the induction proves \eqref{eq:uniform-joint}.
\end{proof}

When only a few arms are close to optimal, we use a sharper bound
involving the sum of the inverse squared gaps of the remaining,
easier arms. The next lemma, analogous to
\cite[lemma B.6]{CLQ17}, controls the probability of an invalid Eliminate call before the
algorithm reaches the scale needed to distinguish the difficult arms.

Order the arms so that
\[
 \mu_{[1]}=m>\mu_{[2]}\ge\cdots\ge\mu_{[K]},
 \qquad
 \Delta_{[i]}=m-\mu_{[i]}.
\]
Ties among the suboptimal arms may be ordered arbitrarily.

\begin{lemma}\label{lem:uniform-local-error}
Fix \(2\le s\le K\). If \(s=2\), set \(r_*=\infty\) and impose no restriction on \(B\). If
\(s\ge3\), assume
\[
 B<\Delta_{[s-1]}^{-2}
\]
and set
\[
 r_*=\left\lfloor\log_2\frac{1}{\Delta_{[s-1]}}\right\rfloor.
\]
Then
\begin{equation}\label{eq:uniform-local}
 \Pp(\text{an invalid Eliminate call occurs in some round }r<r_*)
 \le
 Cs\rho\frac{\sum_{i=s}^K\Delta_{[i]}^{-2}}{B}.
\end{equation}
\end{lemma}
\begin{proof}
Stop the analysis after the first invalid Eliminate call, including the round in which that call
occurs. Set
\[
 q=\left\lfloor\log_2\frac{1}{\Delta_{[s]}}\right\rfloor,
 \qquad
 H_{\mathrm{tail}}=\sum_{i=s}^K\Delta_{[i]}^{-2}.
\]
Call the arms of ranks \(1,\ldots,s-1\) the difficult arms and those of ranks
\(s,\ldots,K\) the tail arms.

\medskip
\noindent\emph{Step 1: Separate the difficult and tail arms.}
For \(r<r_*\),
\[
 \Delta_{[s-1]}\le\frac{\varepsilon_r}{2}.
\]
When \(s=2\), this remains true because \(\Delta_{[1]}=0\). The corrected Mean output satisfies
\[
 \widehat\mu_r\le m+\frac{\varepsilon_r}{8}.
\]
Consequently, the upper mean thresholds in Fraction and Eliminate are at most
\[
 \widehat\mu_r-\frac{9}{8}\varepsilon_r\le m-\varepsilon_r,
 \qquad
 \widehat\mu_r-\frac{5}{8}\varepsilon_r
 \le m-\frac{\varepsilon_r}{2}.
\]
Every active difficult arm therefore lies above both upper thresholds.

The definition of \(q\) gives
\[
 \Delta_{[s]}>2^{-(q+1)}.
\]
Hence, if \(r\ge q+2\), then
\[
 \Delta_{[s]}>2\varepsilon_r.
\]
When Median succeeds in such a round, the corrected Mean output satisfies
\[
 \widehat\mu_r\ge m-\frac{\varepsilon_r}{4}.
\]
The lower mean thresholds in Fraction and Eliminate are therefore at least
\[
 m-2\varepsilon_r
 \qquad\text{and}\qquad
 m-\varepsilon_r,
\]
respectively. Every tail arm lies below both lower thresholds.

\medskip
\noindent\emph{Step 2: Control the calls after a successful Median round.}
Suppose that the first successful Median call with index at least \(q+2\) occurs in a round
\(r<r_*\). Let \(\gamma\) be the fraction of the current active set consisting of tail arms.

If Fraction returns False, its validity gives
\[
 \gamma<\theta_r,
\]
because every tail arm is below the lower Fraction threshold. The active set is unchanged.
In any later round \(v<r_*\), every arm below the upper Fraction threshold is a tail arm, whereas
\[
 \theta_{v-1}\ge\theta_r.
\]
It follows inductively that True is invalid in every such round. All subsequent Fraction outputs
before \(r_*\) are therefore False.

If Fraction returns True and the resulting Eliminate call is valid, every surviving tail arm is
below the lower elimination threshold. The fraction of tail arms in the returned set is therefore
at most \(0.1\). Since \(\theta_{v-1}\ge0.3\) in every later round \(v\),
the same induction shows that all subsequent Fraction outputs before
\(r_*\) are False.

Thus, after the first successful Median call in a round \(r\ge q+2\), either the Eliminate call in
that round is invalid or no further Eliminate call occurs before \(r_*\).

\medskip
\noindent\emph{Step 3: Bound the total weight of the relevant rounds.}
Let \(\mathcal R\) be the collection of entered rounds before \(r_*\), stopped after and including
the first successful Median round with index at least \(q+2\), if such a round occurs. Step 2 shows
that every possible first invalid Eliminate call before \(r_*\) occurs in a round belonging to
\(\mathcal R\).

Apply the localized part of \zcref{lem:uniform-charging} with cutoff \(r_0=r_*\), designating the
top \(s-1\) arms and assigning them artificially to group \(q\). Step 1 verifies the required
upper-threshold condition. Every tail arm has gap at least \(\Delta_{[s]}\), so its gap-group index
is at most \(q\). The lemma therefore gives
\[
 \E\sum_{r\in\mathcal R}|S_r|4^r
 \le C\bigl[H_{\mathrm{tail}}+(s-1)4^q\bigr].
\]
Since
\[
 4^q\le\Delta_{[s]}^{-2}\le H_{\mathrm{tail}},
\]
we obtain
\[
 \E\sum_{r\in\mathcal R}|S_r|4^r
 \le CsH_{\mathrm{tail}}.
\]

\medskip
\noindent\emph{Step 4: Sum the probabilities of the first invalid call.}
Before the first invalid Eliminate call, the optimal arm remains active. Conditional on the past,
the probability that the call in round \(r\) is invalid is therefore at most
\[
 \beta_r=\frac{4\rho|S_r|4^r}{B}.
\]
Summing over the possible locations of the first invalid call and using the preceding estimate,
\[
 \begin{aligned}
 \Pp(\text{an invalid Eliminate call before round }r_*)
 &\le
 \frac{4\rho}{B}
 \E\sum_{r\in\mathcal R}|S_r|4^r \le
 \frac{Cs\rho H_{\mathrm{tail}}}{B}.
 \end{aligned}
\]
Substituting the definition of \(H_{\mathrm{tail}}\) proves \eqref{eq:uniform-local}.
\end{proof}

\subsection{Summing errors across budget scales}

For each suboptimal arm, write
\[
 w_i=\Delta_i^{-2}.
\]
At a stage with budget \(B\), define
\[
 h=1+|\{i\ne *:w_i>B\}|,
 \qquad
 x=\frac1B\sum_{\substack{i\ne *\\w_i\le B}}w_i.
\]
Call the optimal arm and the \(h-1\) suboptimal arms with \(w_i>B\) \emph{hard}; call the remaining
suboptimal arms \emph{easy}. Thus \(Bx\) is the sum of the inverse squared gaps of the easy arms.

\medskip
\noindent\emph{Step 1: A joint bound for the hard arms.}
Suppose that a hard suboptimal arm \(i\in G_q\) is removed in an executed round \(r\). Its weight
satisfies
\[
 B<w_i<4^{q+1}.
\]
Moreover, the first budget test and \(|S_r|\ge2\) give
\[
 8\cdot4^r
 \le4|S_r|4^r
 <B
 <w_i
 <4^{q+1}.
\]
It follows that \(r<q\), so every removal of a hard suboptimal arm is premature. Every removal of
the optimal arm is premature by definition.

If a stage returns an incorrect arm, then the optimal arm and at least \(h-2\) of the \(h-1\) hard
suboptimal arms must have been removed. For \(h\ge2\), take a union bound over the hard suboptimal
arm that might remain. Applying \eqref{eq:uniform-joint} to each resulting collection of \(h-1\)
arms gives
\begin{equation}\label{eq:uniform-J}
 J_1=\rho,
 \qquad
 J_h=(h-1)\rho^{h-1}\quad(h\ge2)
\end{equation}
as upper bounds for the stage's wrong-output probability.

\medskip
\noindent\emph{Step 2: A bound using the easy arms.}
Suppose \(h<K\), so at least one easy arm is present. The hard suboptimal arms have ranks
\(2,\ldots,h\), while the easy arms have ranks \(h+1,\ldots,K\). Apply
\eqref{eq:uniform-local} with \(s=h+1\); when \(h=1\), use its \(s=2\) case.

For \(h\ge2\), the cutoff \(r_*\) in that lemma is the gap-group index of the hard suboptimal arm
with the largest gap. The calculation in Step 1 shows that every executed Eliminate call occurs
before this cutoff. Since a wrong output requires the optimal arm to be removed, it requires an
invalid Eliminate call before \(r_*\). The same conclusion holds for \(h=1\), for which
\(r_*=\infty\). Because
\[
 \sum_{i=h+1}^K\Delta_{[i]}^{-2}=Bx,
\]
the local bound gives
\begin{equation}\label{eq:uniform-L}
 \Pp(\text{wrong output in the stage})
 \le C(h+1)\rho x.
\end{equation}

If \(h=K\), every suboptimal arm satisfies \(\Delta_i<B^{-1/2}\). In any executed round,
\[
 8\cdot4^r<B,
\]
and hence \(\varepsilon_r=2^{-r}>B^{-1/2}\). The upper Fraction threshold is at most
\(m-\varepsilon_r\), which is below every arm mean. Thus True is not valid, every corrected
Fraction output is False, and the stage cannot accept. Its wrong-output probability is therefore
zero.

\medskip
\noindent\emph{Step 3: A necessary condition for acceptance.}
Assign truncated weights
\[
 v_*=1,
 \qquad
 v_i=\min\{w_i/B,1\}\quad(i\ne *),
\]
and write
\[
 v_{\mathrm{tot}}=\sum_i v_i=h+x.
\]
Every accepted stage satisfies
\begin{equation}\label{eq:uniform-certificate}
 \sum_i v_i(1-M_i)\le2.
\end{equation}

To verify this, consider a suboptimal arm \(i\in G_q\) removed in a
round \(r\ge q\). Then
\[
 Bv_i\le w_i<4^{q+1}\le4^{r+1}.
\]
The total \(Bv_i\) over all such removals in round \(r\) is therefore at most
\[
 |S_r|4^{r+1}=4|S_r|4^r,
\]
the increment of \(U\) in that round. The first budget test ensures that the sum of these increments
over all Eliminate calls is less than \(B\). Hence the total weight of suboptimal arms removed at or after their
gap-group indices is less than one.

The single returned arm contributes at most one more. If the output is wrong, then \(M_*=1\), so
the original optimal arm contributes zero to the left side of
\eqref{eq:uniform-certificate}. If the output is correct, the optimal arm is the returned arm.
This proves \eqref{eq:uniform-certificate}.

\medskip
\noindent\emph{Step 4: An exponential-moment bound for acceptance.}
Set
\[
 u=\log(1/\rho).
\]
By choosing the absolute constant \(c\) sufficiently small, we may assume \(u\ge4\). Since
\(M_i\in\{0,1\}\),
\[
 \exp\left(u\sum_i v_iM_i\right)
 =\prod_i\left[1+(e^{uv_i}-1)M_i\right].
\]
All coefficients in this product expansion are nonnegative. Applying
\eqref{eq:uniform-joint} to each product of indicators gives
\[
 \E\exp\left(u\sum_i v_iM_i\right)
 \le\prod_i\left[1+\rho(e^{uv_i}-1)\right].
\]
For \(0\le v\le1\), convexity gives
\[
 e^{uv}-1\le v(e^u-1).
\]
Therefore
\[
 \begin{aligned}
 \E\exp\left(u\sum_i v_iM_i\right)
 &\le
 \exp\left(\rho(e^u-1)\sum_i v_i\right)\\
 &\le e^{v_{\mathrm{tot}}},
 \end{aligned}
\]
because \(\rho(e^u-1)=1-\rho\le1\).

On acceptance, \eqref{eq:uniform-certificate} implies
\[
 \sum_i v_iM_i\ge v_{\mathrm{tot}}-2.
\]
If \(v_{\mathrm{tot}}\ge4\), Markov's inequality consequently gives
\begin{equation}\label{eq:uniform-accept}
 \begin{aligned}
 \Pp(\text{accept})
 &\le e^{v_{\mathrm{tot}}}\rho^{v_{\mathrm{tot}}-2}\\
 &\le\rho^{v_{\mathrm{tot}}/4}.
 \end{aligned}
\end{equation}
The final inequality uses \(u=\log(1/\rho)\ge4\). Since every incorrect output is an acceptance,
\eqref{eq:uniform-accept} also bounds the wrong-output probability.

The preceding exponential-moment argument follows
\cite[lemmas B.11--B.12 and section B.5]{CLQ17}, with the truncated weights above.

\medskip
\noindent\emph{Step 5: Sum over the geometric budget sequence.}
For a fixed value \(h<K\), the corresponding stages form an interval in the sequence
\(B_t=100^t\). Throughout this interval the easy-arm set is fixed, and each successive stage
replaces \(x\) by \(x/100\). We divide the stages into three ranges.

If \(x<1\) and \(h=1\), summing \eqref{eq:uniform-L} over the geometric sequence of \(x\)'s gives
a contribution of at most \(C\rho\). If \(x<1\) and \(h\ge2\), then
\[
 \min\{J_h,C(h+1)\rho x\}
 \le\sqrt{J_hC(h+1)\rho x}
 \le Ch\rho^{h/2}\sqrt{x}.
\]
For fixed \(h\), the sum of \(\sqrt{x}\) over these stages is at most \(10/9\). Their total
contribution is therefore at most \(Ch\rho^{h/2}\), and summing over \(h\ge2\) gives \(O(\rho)\).

If \(1\le x<16\), there is at most one stage for each \(h\), because successive positive values of
\(x\) differ by a factor of \(100\). Hence \eqref{eq:uniform-J} gives a total contribution bounded
by
\[
 J_1+\sum_{h\ge2}J_h=O(\rho).
\]

Finally, suppose \(x\ge16\). For each fixed \(h\), enumerate these stages backward from the last one
in this range. At the \(j\)th preceding stage,
\[
 x\ge16\cdot100^j.
\]
Using \(v_{\mathrm{tot}}=h+x\) in \eqref{eq:uniform-accept}, the total contribution of this range is
at most
\[
 \sum_{h\ge1}\rho^{h/4}
 \sum_{j\ge0}\rho^{4\cdot100^j}
 =O(\rho).
\]

Stages with \(h=K\) contribute zero by Step 2. Generate the independent
data for every stage in advance, so that the output of each stage is
defined even if that stage is never reached. The event that the
comparison process returns an incorrect arm is contained in the union
of these stagewise wrong-output events. The preceding bounds and the
union bound therefore give probability at most \(C\rho\). Combining
this estimate with \eqref{eq:uniform-coupling}, the error 
probability of the actual algorithm is at most
\[
 C\eta+C\rho=C\eta+C\eta^2.
\]
Since \(\eta=c\delta\), choosing the absolute constant \(c>0\) sufficiently small makes this
quantity at most \(\delta\).

\section{The expected sample complexity}\label{sec:expected-cost}

It remains to bound the unconditional expected stopping time \(\E T\). We first show that the
outer Mean and Fraction calls use only a fixed fraction of each stage budget. We then show that,
once the budget is sufficiently large, a stage accepts with probability at least \(199/200\).
Together with the deterministic sample limit and the geometric growth of the stage budgets, this
will control the expected cost of all restarts.

\subsection{The cost of the outer tests}

The increment \(\theta_r-\theta_{r-1}\) is indexed by
\(j=R_t+1-r\), the number of permitted rounds from \(r\) through
\(R_t\). After factoring out \(4^{R_t}\), its polynomial growth is
dominated by the resulting geometric decay. The series is summable
independently of the active sets.

\begin{lemma}\label{lem:uniform-aux}
There is an absolute integer \(J\) such that, in every stage \(t\) and on every sample path, the total
number of samples used by the outer Mean and Fraction calls is at most \(B_t/100\).
\end{lemma}

\begin{proof}
For \(1\le r\le R_t\), we have
\[
 R_t\le\log_4 B_t=(\log_4 100)t.
\]
Since
\[
 \alpha_{t,r}=\frac{\eta}{50t^2r^2},
\]
it follows that
\[
 \log\frac1{\alpha_{t,r}}
 =a+\log 50+2\log t+2\log r
 \le C\bigl[a+\log(t+1)\bigr]
 =CL_t.
\]

Fix a permitted round \(r\) and put
\[
 j=R_t+1-r.
\]
The Mean call has accuracy \(\varepsilon_r/8\), so its cost is at most $
 C4^rL_t$. 
The two mean thresholds in the Fraction call differ by $
 \frac58\varepsilon_r$, 
while its fraction thresholds differ by
\[\theta_r-\theta_{r-1}=\frac1{10j^2}.\]
The Fraction cost is therefore at most
$C4^rL_t\,j^4\log(10j^2)$. 
Since \(j\ge1\), the combined cost of the two outer calls in round \(r\) is at most
\[
 C4^rL_tj^6.
\]

Summing over all permitted rounds, including rounds the stage may not reach, gives
\[
 \begin{aligned}
 \text{total outer-test cost}
 &\le
 CL_t\sum_{r=1}^{R_t}4^r(R_t+1-r)^6\\
 &=
 CL_t4^{R_t}
 \sum_{j=0}^{R_t-1}4^{-j}(j+1)^6\\
 &\le
 CL_t4^{R_t}
 \sum_{j=0}^{\infty}4^{-j}(j+1)^6.
 \end{aligned}
\]
The last series is finite and yields an absolute constant.

Whenever \(R_t\ge1\), its definition gives
\[
 L_t4^{R_t}\le4^{-J}B_t.
\]
Choose the absolute integer \(J\) sufficiently large that the resulting constant multiple of
\(4^{-J}B_t\) is at most \(B_t/100\). If \(R_t=0\), the stage rejects before making any outer call,
so its outer-test cost is zero.
\end{proof}

\subsection{Acceptance at a sufficient budget}

The next lemma fixes the absolute multiplier \(M\) in the deterministic sample limit. It also
introduces an absolute constant \(A\) used only in the analysis.

\begin{lemma}\label{lem:uniform-success}
Define
\[
 F=H(a+\Ent(I))+D\ell_g,
 \qquad
 \ell_g=\log\log(e^e/g).
\]
There are absolute constants \(M\) and \(A\) such that every stage with \(B_t\ge AF\) accepts with
probability at least \(199/200\).
\end{lemma}

\begin{proof}
Write \(B=B_t\) and \(R=R_t\).

\medskip
\noindent\emph{Step 1: Control rejection at the sample limit.}
Temporarily remove the deterministic sample limit, retaining both budget tests and the restriction
to \(R\) rounds. On every executed round, \(w_r<B/4\). If Eliminate is called, then
\[
 \beta_r=\frac{4\eta^2w_r}{B},
\]
and
\[
 \log\frac1{\beta_r}
 =\log\frac{B}{4\eta^2w_r}
 \le2\log\frac{B}{\eta w_r}.
\]
The conditional expected Eliminate cost in round \(r\) is therefore at most \(Cz_r\). The Median
cost is at most \(Cw_r\le Cz_r\). Since
\[
 \sum_{\text{executed }r}z_r<100B,
\]
the tower property bounds the expected total cost of all Median and Eliminate calls by \(CB\).
By \zcref{lem:uniform-aux}, the outer Mean and Fraction calls use at most \(B/100\) additional
samples. Thus the expected cost of the uncapped stage is at most $
 C_0B$ 
on every instance, including paths on which subroutine errors occur.

Choose the absolute constant \(M\) sufficiently large. Markov's inequality then shows that the
uncapped stage uses more than \(MB\) samples with probability at most \(1/2000\). Coupling the
capped and uncapped stages until the cap is reached shows that imposing the deterministic limit
causes rejection with probability at most \(1/2000\). The same estimate holds for the comparison
process, since the preceding cost bound applies equally after correcting the outer outputs.

\medskip
\noindent\emph{Step 2: Control rejection at the budget tests.}
In the comparison process, an Eliminate call in round \(r\) is invalid with conditional probability
at most \(\beta_r\). Consequently, \eqref{eq:uniform-riskbudget} and the tower property give
\[
 \Pp(\text{some Eliminate call is invalid})\le\rho.
\]

Stop the comparison process at its first invalid Eliminate call, but include the round containing
that call and any round rejected by a budget test. If the \(U\)-test rejects, then
\[
 \sum_r c_r\ge\frac{B}{4}.
\]
If the \(V\)-test rejects, then
\[
 \sum_r c_r\log\left(e+\frac{B}{\eta c_r}\right)\ge100B.
\]
Indeed, the corresponding sum without the \(e\) is \(V_r+z_r\), which is at least \(100B\).

For \(B\ge H\), Markov's inequality and \eqref{eq:uniform-charges} therefore give
\begin{equation}\label{eq:uniform-guardprob}
 \Pp(\text{a budget test rejects before the first invalid Eliminate call})
 \le
 \frac{CH}{B}
 \left[a+\Ent(I)+1+\log\frac{B}{H}\right].
\end{equation}

To bound the right side, put
\[
 \sigma=a+\Ent(I),
 \qquad
 y=\frac{B}{H}.
\]
Since \(F\ge H\sigma\), the condition \(B\ge AF\) gives \(y\ge A\sigma\). The function
\[
 y\longmapsto\frac{\sigma+1+\log y}{y}
\]
is decreasing for \(y\ge1\). Hence the right side of \eqref{eq:uniform-guardprob} is at most
\[
 C\frac{\sigma+1+\log(A\sigma)}{A\sigma}
 \le C\frac{1+\log A}{A}.
\]
It can therefore be made arbitrarily small by choosing the absolute constant \(A\) sufficiently
large.

\medskip
\noindent\emph{Step 3: Ensure sufficiently many rounds.}
Set
\[
 f=\frac{F}{D}.
\]
Since \(H\ge D\),
\[
 f=\frac{H}{D}[a+\Ent(I)]+\ell_g\ge a+\ell_g.
\]
Let \(t_0\) be the first stage index for which \(B_{t_0}\ge AF\). By minimality,
\[
 B_{t_0}<100AF=100ADf,
\]
and therefore
\[
 t_0+1
 \le
 2+\frac{\log(100A)+\log D+\log f}{\log100}.
\]
The definition of \(\ell_g\) gives
\[
 \log D=2\log(1/g)=2(e^{\ell_g}-e)\le2e^f.
\]
Since \(f\ge1\), we also have \(\log f\le e^f\). It follows that
\[
 t_0+1\le C[1+\log A+e^f].
\]
Using \(a\le f\), we obtain
\[
 \begin{aligned}
 L_{t_0}
 &=a+\log(t_0+1)\le2f+C+\log(1+\log A)
 \le C_Af,
 \end{aligned}
\]
where
\[
 C_A=O(1+\log\log(A+e)).
\]
Consequently,
\[
 \frac{B_{t_0}}{L_{t_0}}
 \ge\frac{A}{C_A}D.
\]
The ratio \(B_t/L_t\) is increasing in \(t\), so the same bound holds at every later stage.

Fix an integer \(b\ge1\). Because \(A/C_A\to\infty\) as \(A\to\infty\), we may choose \(A\) large
enough that every stage with \(B_t\ge AF\) satisfies
\[
 R_t\ge Q+2+b.
\]
Consider the comparison process without the deterministic sample limit.
Suppose that no budget test rejects and no Eliminate call is invalid.
If the active set has not become a singleton by the last permitted round,
then every Median call in rounds
\[
 Q+2,Q+3,\ldots,R_t
\]
must fail. Each call uses fresh samples and has conditional failure
probability at most \(\kappa\). Since the list contains at least \(b+1\)
calls, successive conditioning gives
\[
 \Pp\left(
 \begin{gathered}
  \text{no singleton by the last permitted round,}\\
  \text{no budget test rejects, and}\\
  \text{no Eliminate call is invalid}
 \end{gathered}
 \right)
 \le\kappa^{b+1}.
\]

\medskip
\noindent\emph{Step 4: Combine the rejection probabilities.}
First choose \(b\) so that \(\kappa^{b+1}\) is sufficiently small. Then choose \(A\) large enough
both to obtain the required number of rounds and to make
\eqref{eq:uniform-guardprob} sufficiently small. These choices can ensure
that, in the comparison process without the deterministic sample limit,
the probability of rejection before the first invalid Eliminate call,
either at a budget test or at the final round, is at most \(1/2000\).

The probability of an invalid Eliminate call is at most \(\rho\). By
\eqref{eq:uniform-coupling}, the within-stage probability that the actual and comparison processes
disagree is at most \(C\eta/t^2\). Choose the absolute constant \(c\) sufficiently small that
\[
 \rho\le\frac1{2000},
 \qquad
 \frac{C\eta}{t^2}\le\frac1{2000}
\]
for every \(t\ge1\). Step 1 bounds the probability of rejection at the sample limit by
\(1/2000\). Summing these four contributions gives a rejection probability below \(1/200\).
Thus every stage with \(B_t\ge AF\) accepts with probability at least \(199/200\).
\end{proof}

\subsection{Summing the stage costs}

Let \(t_*\) be the first index such that
\[
 B_{t_*}\ge AF.
\]
By minimality,
\[
 B_{t_*}<100AF.
\]
Every stage uses fresh data and begins with all \(K\) arms. Hence
\zcref{lem:uniform-success}, applied conditionally after each rejection, gives
\[
 \Pp(\text{the algorithm reaches stage }t_*+k)\le200^{-k},
 \qquad k\ge0.
\]
Since stage \(t\) uses at most \(\lceil MB_t\rceil\) samples,
\[
 \begin{aligned}
 \E T
 &\le
 \sum_{t<t_*}\lceil MB_t\rceil
 +\sum_{k\ge0}200^{-k}\lceil M100^kB_{t_*}\rceil\\
 &\le CB_{t_*}
 \le CF.
 \end{aligned}
\]
The reach probabilities tend to zero, so the algorithm also stops almost surely.

Finally,
\[
 a=\log(1/\eta)=\log(1/\delta)+\log(1/c).
\]
Because \(c\) is absolute and \(\delta<0.01\), \(a\) is bounded by an absolute constant multiple of
\(\log(1/\delta)\). Therefore
\[
 \E T
 \le
 C\left\{
 H(\log(1/\delta)+\Ent(I))
 +D\log\log(e^e/g)
 \right\}.
\]
Together with the correctness bound from the preceding section, this proves
\zcref{thm:uniform}.

\begin{proof}[Proof of \zcref{cor:almost}]
Apply the lower bound in \zcref{thm:instance} to
$H(\log(1/\delta)+\Ent(I))$ in \zcref{thm:uniform}. 
\end{proof}

\appendix
\zcsetup{countertype={section=appendix}}
\section{Conventions in the original conjectures}\label{app:conventions}
Our formulation follows that of \cite[section 1 and definition 3.1]{CL16} with minor convention differences. We record these below and explain how \zcref{thm:instance,thm:uniform} still establish
\cite[conjectures 3.2 and 3.5]{CL16} as formulated therein.

\subsection{Gap-group boundary conventions}
Choose a suboptimal arm $i$ at random with probability $\Delta_i^{-2}/H$. Let $R$ be the index of its gap group under our convention and $S$ the corresponding index under
the convention $[2^{-s},2^{-s+1})$ in~\cite{CL16}, including the group containing the endpoint $1$. Every cell of either partition meets at most two cells of the
other. Hence
\[
 \Ent(S\mid R)\le\log2,\qquad
 \Ent(R\mid S)\le\log2.
\]
The entropy chain rule gives
\[
 |\Ent(R)-\Ent(S)|\le\log2.
\]
Thus the contributions $H\Ent(I)$ differ by an additive $O(H)$, absorbed by $H\log(1/\delta)$. The
definition of the gap groups in~\cite[definition 1.7]{CLQ17} already agrees with ours.

\subsection{Iterated logarithm}
For $x=\log(1/g)\ge0$,
\[
 0\le\log(e+x)-\log\max\{e,x\}\le\log2.
\]
Our shifted iterated logarithm and the clipped form used in \cite{CL16} therefore differ by at most a constant. Since
$D\le H$, the resulting $O(H)$ change is absorbed by $H\log(1/\delta)$. The unregularized logarithm
is undefined at $g=1$ and can be negative for large gaps; clipping makes the sample-complexity
convention explicit.

\subsection{Stopping conventions}
The correctness definition in \cite{CL16} only requires a correct output with probability at least $1-\delta$, not that the algorithm also almost surely stops.
The algorithms proving the upper bounds stop almost
surely on every instance in $\mathcal S_K$, and therefore satisfy either correctness definition. The lower-bound proof does not require the algorithm to stop on the Gaussian alternative instances, as it uses only their correct-output probabilities at finite horizons. 
If the expected count on the target instance is infinite, its lower bound is immediate. Otherwise the algorithm stops almost surely on that instance, permitting the limits in the lower-bound proof. Thus the bounds also cover the original correctness convention.

These comparisons preserve all bounds up to absolute constants. Together with \zcref{cor:almost},
they give the stated formulations of the original conjectures.

\section{Sampling subroutines}\label{app:primitives}
We prove the subroutine guarantees used in \zcref{sec:algorithm} for $1$-sub-Gaussian reward
laws. Input sets may have tied optimal arms. Different calls use fresh samples. Within Fraction, arm identities are sampled independently across draws; conditional on
these identities, sample blocks are independent. Fraction samples are also independent of
the estimation blocks used to delete arms. All sample sizes are rounded up to integers. Accuracy
parameters and differences between mean thresholds are at most one, so rounding is absorbed in the
cost bounds.

\subsection{Mean estimation}
For independent samples $X_1,\ldots,X_n$ from a $1$-sub-Gaussian arm with mean $\mu$,
independence gives
\[
 \E\exp\left(\lambda\sum_{s=1}^n(X_s-\mu)\right)\le e^{n\lambda^2/2}.
\]
Writing $\overline X_n=n^{-1}\sum_{s=1}^nX_s$ and applying the Chernoff bound with
$\lambda=x$ and $\lambda=-x$ yields, for $x>0$,
\begin{equation}\label{eq:subgaussian-tail}
 \Pp(\overline X_n-\mu\ge x)\le e^{-nx^2/2},\qquad
 \Pp(\overline X_n-\mu\le-x)\le e^{-nx^2/2}.
\end{equation}
This is the standard sub-Gaussian sample-mean bound; see \cite[corollary 5.5]{LS20}.
For accuracy $e$ and error level $\alpha$, take
\[
 n(e,\alpha)=\left\lceil2e^{-2}\log(2/\alpha)\right\rceil
\]
samples and return their average. Equation \eqref{eq:subgaussian-tail} gives
\[
 \Pp(|\widehat\mu-\mu|>e)\le2e^{-ne^2/2}\le\alpha.
\]
The count is deterministic and of order $e^{-2}\log(1/\alpha)$ on the stated parameter range.
Fresh samples give the same bound conditional on the history before a call.

\subsection{Median elimination}
We use the median-elimination construction of \cite{EMM06}, in the form stated in \cite[fact
5.2]{CLQ17}. Start with $S_1=S$. While $|S_j|\ge2$, put
\[
 e_j=\frac e4\left(\frac34\right)^{j-1},\qquad
 a_j=0.01\cdot2^{-j}.
\]
Estimate each current mean to accuracy $e_j/2$ with individual failure probability at most $a_j/3$,
and retain the $\lceil|S_j|/2\rceil$ arms with the largest estimates, using fixed tie-breaking.
Return the remaining arm.

Conditional on the current set, fix one optimal arm. Its estimate is inaccurate with probability at
most $a_j/3$. If it is accurate and no arm whose mean is within $e_j$ of its mean survives, then at
least $\lceil|S_j|/2\rceil$ retained arms have inaccurate estimates. The expected number of
inaccurate estimates is at most $|S_j|a_j/3$, so Markov's inequality bounds the probability of this
event by $2a_j/3$. Thus the maximum true mean drops by more than $e_j$ with conditional probability
at most $a_j$. As
\[
 \sum_{j\ge1}e_j=e,\qquad \sum_{j\ge1}a_j=0.01,
\]
the returned arm has mean within $e$ of the original maximum with probability at least $0.99$.

Until termination, $|S_j|\le2|S|2^{-(j-1)}$. The deterministic total cost is bounded by
\[
 C|S|e^{-2}\sum_{j\ge1}\left(\frac89\right)^{j-1}(j+1)
 =O(|S|e^{-2}).
\]
Fresh samples give the same guarantee conditional on an arbitrary history before the call.

\subsection{Fraction testing}
Write $\varepsilon=u-l$ and $d=\theta_+-\theta_-$. Draw
\[
 m_F=\left\lceil36d^{-2}\log(2/\alpha)\right\rceil
\]
independent uniform arm identities from $S$, with replacement. For each draw, obtain a fresh mean
estimate of accuracy $\varepsilon/2$ and error level $d/6$. Let $Y_j$ indicate that this estimate is
below $(l+u)/2$. Return True precisely when
\[
 m_F^{-1}\sum_{j=1}^{m_F}Y_j>(\theta_-+\theta_+)/2.
\]
The variables $Y_j$ are independent and identically distributed. If the fraction of arms with mean
below $l$ is at least $\theta_+$, then
\[
 \E Y_j\ge\theta_+-d/6.
\]
If the fraction with mean below $u$ is at most $\theta_-$, then
\[
 \E Y_j\le\theta_-+d/6.
\]
In the first case, False is an invalid answer; in the second, True is invalid. Either invalid answer
requires a deviation of at least $d/3$ from the Bernoulli mean. Hoeffding's inequality bounds its
probability by
\[
 \exp(-2m_Fd^2/9)\le\alpha/2.
\]
These implications prove the stated True/False guarantees. The deterministic number of samples is
\[
 O\left(\varepsilon^{-2}\log(1/\alpha)d^{-2}\log(1/d)\right)
\]
for $0<\varepsilon\le1$, $0<d\le0.1$, and $0<\alpha<0.1$. This is Algorithm 2 and Fact 5.3
of~\cite{CLQ17} with explicit rounding.

\subsection{Elimination}\label{app:elimination}
We use the elimination procedure in \cite[algorithm 3 and fact 5.4]{CLQ17}, returning immediately if
the active set is empty and retaining arms whose estimates equal the deletion threshold. Fraction testing and deletion use separate samples. Put
\[
 \varepsilon=u-l,\qquad v=(l+u)/2,\qquad
 q_j=\frac{\alpha}{10\cdot2^j}\quad(j\ge1).
\]
Start with $S_1=S$. At internal stage $j$:
\begin{enumerate}[leftmargin=*]
\item If $S_j$ is empty, return it. Otherwise call Fraction with mean thresholds $l,v$, fraction thresholds $0.05,0.1$, and error level $q_j$.
\item If False, return $S_j$.
\item If True, estimate each current arm mean independently to accuracy $\varepsilon/4$ with error level $q_j$, and retain precisely the arms whose estimates are at least $v+\varepsilon/4$. Continue with the retained set.
\end{enumerate}

For the cost calculation only, set the active-set size to zero after
the procedure returns.
Conditional on a current set, a valid True answer implies that more than $0.05|S_j|$ arms have mean below $v$. Each such arm survives the independent
estimation step with probability at most $q_j$. A False answer makes the number of active arms in the next internal stage
zero. On an invalid Fraction answer, the number of active arms in the next internal stage is still at most $|S_j|$. Averaging over
this event, whose probability is at most $q_j$, gives
\[
 \E[|S_{j+1}|\mid S_j]
 \le[1-0.05(1-q_j)^2]|S_j|\le0.96|S_j|.
\]
The contraction holds without conditioning on correctness. 
Iterating it gives
\[
 \E|S_j|\le0.96^{j-1}|S|.
\]
The conditional stage cost is at most
\[
 C|S_j|\varepsilon^{-2}[\log(1/\alpha)+j].
\]
Here the Fraction cost is absorbed because $|S_j|\ge1$ while active. Summing the expected
internal-stage costs bounds the unconditional expected total number of samples by
\[
 C|S|\varepsilon^{-2}\sum_{j\ge1}0.96^{j-1}[\log(1/\alpha)+j]
 =O(|S|\varepsilon^{-2}\log(1/\alpha)).
\]
The same estimate gives
\[
 \Pp(\text{still active at internal stage }j)
 \le\E|S_j|\le0.96^{j-1}|S|.
\]
This tends to zero as $j\to\infty$, so the subroutine terminates almost surely.

We next verify the fraction guarantee and bound the probability that specified arms of mean
at least $u$ are all removed. The probability that any internal Fraction answer is
invalid is at most $\sum_jq_j=\alpha/10$. If every such answer is valid, a nonempty set is returned only after a False answer and therefore has
low-mean fraction less than $0.1$. The empty set satisfies the non-strict inequality automatically.
For each input arm, generate independent estimation blocks in advance for all internal stages,
separately from the Fraction samples. If an arm has mean at least $u$, every accurate estimate
is at least
\[
 u-\varepsilon/4=v+\varepsilon/4,
\]
so the arm cannot be removed unless one of its own estimates is inaccurate. For each arm, the event
that at least one estimate is inaccurate has probability at most $\alpha/10$. These events are
independent across arms. Consequently the probability that all $k$ specified arms with mean at least
$u$ are removed is at most $(\alpha/10)^k\le\alpha^k$. This argument based on independent estimation
errors follows \cite[lemma B.9]{CLQ17} and permits adaptive use of the blocks.

For one designated optimal arm of mean at least $u$, the probability of losing it or violating the
bound on the number of returned arms with mean below $l$ is at most $\alpha/5<\alpha$. Fresh samples give these guarantees
conditional on any history before the call.

\bibliographystyle{plain}
\bibliography{gap_entropy}

\end{document}